\documentclass[10pt]{article} 
\usepackage[preprint]{tmlr}

\usepackage{amsmath,amsfonts,bm}

\def\eqref#1{equation~\ref{#1}}

\def\1{\bm{1}}

\DeclareMathAlphabet{\mathsfit}{\encodingdefault}{\sfdefault}{m}{sl}
\SetMathAlphabet{\mathsfit}{bold}{\encodingdefault}{\sfdefault}{bx}{n}

\usepackage{bm}
\usepackage{hyperref}
\usepackage{url}
\usepackage{amsmath, amssymb, amsthm, amsfonts}
\usepackage{array}
\usepackage{algorithm}
\usepackage[noend]{algpseudocode}
\usepackage{natbib}    
\usepackage{caption}   
\usepackage{graphicx}
\usepackage{multirow}
\usepackage{tabularx}
\usepackage{makecell}
\usepackage{tcolorbox}
\tcbuselibrary{listings}
\usepackage{float}
\usepackage{ragged2e}
\usepackage{graphicx}
\usepackage{cuted}
\usepackage{xcolor}    
\usepackage{enumitem}
\usepackage{subcaption}
\usepackage{booktabs}
\usepackage{url}
\usepackage{etoolbox}
\usepackage{wrapfig}
\usepackage{sidecap}
\usepackage{subcaption}

\newtheorem{theorem}{Theorem}[section]
\newtheorem{lemma}[theorem]{Lemma}

\newtheorem{proposition}[theorem]{Proposition}

\title{SWAN: Sparse Winnowed Attention for Reduced Inference Memory via Decompression-Free KV-Cache Compression}

\author{\name Santhosh G S \email santhoshgs013@gmail.com \\
      \addr Centre for Responsible AI\\
      Indian Institute of Technology Madras
      \AND
      \name Saurav Prakash \email saurav@ee.iitm.ac.in \\
      \addr Department of Electrical Engineering\\
      Indian Institute of Technology Madras
      \AND
      \name Balaraman Ravindran \email ravi@dsai.iitm.ac.in\\
      \addr Wadhwani School of Data Science and Artificial Intelligence\\
      Indian Institute of Technology Madras}

\def\month{MM}  
\def\year{YYYY} 
\def\openreview{\url{https://openreview.net/forum?id=XXXX}} 

\begin{document}

\maketitle

\begin{abstract}
Large Language Models (LLMs) face a significant bottleneck during autoregressive inference due to the massive memory footprint of the Key-Value (KV) cache. Existing compression techniques like token eviction, quantization, or other low-rank methods often risk information loss, have fixed limits, or introduce significant computational overhead from explicit decompression steps. In this work, we introduce SWAN, a novel, fine-tuning-free framework that eliminates this overhead. Our method uses an offline orthogonal matrix to rotate and prune the KV-cache, which is then used directly in the attention computation without any reconstruction. Our extensive experiments demonstrate that SWAN, augmented with a small dense buffer, offers a robust trade-off, maintaining performance close to the uncompressed baseline even at aggressive 50-60\% memory savings per-token on KV-cache. A key advantage is its runtime-tunable compression level, allowing operators to dynamically adjust the memory footprint, a flexibility absent in methods requiring fixed offline configurations. This combination of a decompression-free design, high performance under compression, and adaptability makes SWAN a practical and efficient solution for serving LLMs with long contexts.
\end{abstract}

\section{Introduction}
\label{sec:introduction}

The remarkable capabilities of Large Language Models (LLMs) are fundamentally linked to their ability to process extensive context lengths \citep{vaswani2023attentionneed, bai2024longbenchbilingualmultitaskbenchmark}, enabling sophisticated tasks like summarizing long documents. However, this advancement creates a severe inference bottleneck. During autoregressive generation, a Key-Value (KV) cache is essential for storing intermediate attention states to avoid costly recomputation with each new token. For long sequences, the memory required for this cache can vastly exceed that of the model weights, making it the primary performance bottleneck. For instance, a Llama-2 7B model processing a 32k token sequence with a batch size of 16 requires approximately 14\,GB for its weights but a staggering 256\,GB for its KV cache.

To address this bottleneck, researchers have pursued two main directions. \textbf{Architectural methods} such as Grouped-Query Attention (GQA) \citep{ainslie2023gqatraininggeneralizedmultiquery} improve efficiency but require pre-training. \textbf{Post-training approaches} are more flexible but introduce limitations: quantization \citep{hooper2025kvquant10millioncontext} imposes a fixed compression ceiling, while token eviction \citep{zhang2023h2oheavyhitteroracleefficient} risks irreversible information loss, particularly harmful for tasks with long-range dependencies.

A promising direction and the focus of our work leverages a fundamental property of the attention mechanism: the inherent low-rank structure of key and value vectors \citep{singhania2024lokilowrankkeysefficient, saxena2024eigenattentionattentionlowrank}. This redundancy in the standard KV-cache opens the door to significant compression. By first using an orthogonal rotation to concentrate the most salient information into fewer dimensions, we can then prune the vector with minimal information loss. Building on this observation, our contributions are as summarized below:

\begin{itemize}
    \item \textbf{A novel, decompression-free framework for KV-cache compression.} We introduce SWAN (Sparse Winnowed Attention), which performs attention directly on a compressed, sparse KV-cache. This eliminates the conventional reconstruction step, removing a significant source of computational overhead present in some prior low-rank methods.

    \item \textbf{A unified approach for simultaneous memory and compute savings.} By using an offline SVD-derived matrix to rotate and prune KV vectors, SWAN's sparse cache can be multiplied directly with dense queries. This design inherently reduces both the memory footprint and the FLOPs required for the attention inner product.

    \item \textbf{Practical runtime adaptability.} SWAN's pruning threshold is a tunable hyperparameter that can be adjusted at inference time, allowing operators to dynamically balance the trade-off between performance and compression without any offline model modification, a crucial flexibility for real-world serving environments.
    
    \item \textbf{Theoretical and empirical validation.} We provide a theoretical analysis that establishes a clear break-even point for computational savings. We complement this with a detailed empirical analysis on a wide range of NLP benchmarks and model architectures, demonstrating SWAN's effectiveness and robustness in real-world scenarios.
\end{itemize}

\section{Related Works}
\label{related_works}

Efforts to mitigate the KV-cache bottleneck in LLM inference can be broadly categorized into architectural modifications, system-level optimizations, and post-training compression. 

\textbf{Architectural changes} like Multi-Query and Grouped-Query Attention (MQA/GQA) \citep{shazeer2019fasttransformerdecodingwritehead, ainslie2023gqatraininggeneralizedmultiquery} reduce the cache size from the outset but are inapplicable to the vast number of pre-trained models. Orthogonally, system-level solutions like PagedAttention \citep{kwon2023efficientmemorymanagementlarge} optimize memory management but do not reduce the fundamental size of the cache itself. \textbf{Post-training strategies} offer more adaptable solutions. These include Token Eviction methods like StreamingLLM \citep{xiao2024efficientstreaminglanguagemodels} and H2O \citep{zhang2023h2oheavyhitteroracleefficient}, which discard KV pairs but risk a permanent loss of critical information. Quantization techniques such as KVQuant \citep{hooper2025kvquant10millioncontext} and KIVI \citep{https://doi.org/10.13140/rg.2.2.28167.37282} reduce the numerical precision of the cache, but are constrained by a hard upper limit on their compression ratio. \textbf{Low-Rank Approximation}, the most relevant area to our work, leverages the insight that KV vectors occupy a low-dimensional subspace. However, existing methods in this space have notable limitations. Approaches like SparQ Attention \citep{ribar2024sparqattentionbandwidthefficientllm}, AQUA Attention \citep{s2025aquaattentionquerymagnitudes}, and Loki \citep{singhania2024lokilowrankkeysefficient} prioritize computational savings over memory reduction. Others, like Eigen Attention \citep{saxena2024eigenattentionattentionlowrank}, tackle the memory issue but require modifying model weights offline for a fixed compression level, sacrificing crucial runtime flexibility. Frameworks such as Lexico \citep{kim2024lexicoextremekvcache} introduce significant latency by relying on separate compression and decompression steps at every single decoding stage. Our work, SWAN, fills this void with a decompression-free framework that uses a pruned, sparse cache directly in the attention computation. This eliminates reconstruction latency and offers a flexible, runtime-tunable threshold for a balance between performance and compression.

\section{Standard Multi-Headed Attention}
\label{sec:standard_mha }
To establish context for our methodology, we first review the standard self-attention mechanism, which is the core component of Transformer-based LLMs. For a given input sequence of embeddings $X \in \mathbb{R}^{n \times d}$, where $n$ is the sequence length and $d$ is the model dimension, the attention mechanism computes three intermediate representations: the Query ($Q$), Key ($K$), and Value ($V$) matrices. These are derived through linear projections using learned weight matrices $W_Q, W_K, W_V \in \mathbb{R}^{d \times d_h}$, where $d_h$ is the head dimension: $Q = X W_Q, \quad K = X W_K, \quad V = X W_V $. The output of a single attention head is then calculated using the scaled dot-product attention formula \citep{vaswani2023attentionneed}: $\text{Attention}(Q, K, V) = \text{softmax}\left(\frac{QK^T}{\sqrt{d_h}}\right)V$.

In practice, Transformers employ a Multi-Head Attention (MHA) mechanism \cite{vaswani2023attentionneed}, which allows the model to jointly attend to information from different representation subspaces. This is achieved by running the attention mechanism in parallel across $N_h$ independent heads. Each head $j$ has its own set of learned projection matrices $W_Q^{(j)}, W_K^{(j)}, W_V^{(j)} \in \mathbb{R}^{d \times d_h}$. The output of each head is computed as: $ \text{head}_j = \text{Attention}(XW_Q^{(j)}, XW_K^{(j)}, XW_V^{(j)}) $.
The outputs of all heads are then concatenated and projected back to the original model dimension $d$ using an output weight matrix $W_O \in \mathbb{R}^{N_h d_h \times d}$: $\text{MHA}(X) = \text{Concat}(\text{head}_1, \dots, \text{head}_{N_h}) W_O$.
The inference process for an LLM is typically divided into two distinct phases: the prompting (or prefill) phase and the autoregressive decoding phase.

\subsection{Prompting Phase}
In the prompting phase, the model processes the entire input prompt of length $n$ in parallel. For each head, the matrices $Q, K, V \in \mathbb{R}^{n \times d_h}$ are computed for all tokens at once. The attention scores are computed via a large matrix multiplication ($QK^T$), resulting in an attention output for every token in the prompt. This phase is compute-bound due to the large matrix operations. Crucially, at the end of this phase, the computed key and value matrices, $K$ and $V$ for all $N_h$ heads, are stored in memory. This stored data is referred to as the KV-cache.

\subsection{Autoregressive Decoding Phase}
Following the prompt processing, the model generates output tokens one at a time. For each new token $i+1$, a single new query vector $q_{i+1} \in \mathbb{R}^{1 \times d_h}$ is computed for each head from the embedding of the previously generated token $i$. Similarly, new key ($k_{i+1}$) and value ($v_{i+1}$) vectors are computed. These new vectors are then appended to the existing KV-cache for each respective head: $
    K_{\text{cache}}^{(i+1)} = [K_{\text{cache}}^{(i)}; k_{i+1}], \quad V_{\text{cache}}^{(i+1)} = [V_{\text{cache}}^{(i)}; v_{i+1}]
$. 
The attention output for the new token is then calculated using the new query and the entire history stored in the updated cache for each head:
$
    o_{i+1}^{(j)} = \text{softmax}\left(\frac{q_{i+1}^{(j)} (K_{\text{cache}}^{(i+1, j)})^T}{\sqrt{d_h}}\right)V_{\text{cache}}^{(i+1, j)}
$.
This phase is memory-bandwidth bound. At each generation step, the entire KV-cache, which grows with every new token, must be read from high-bandwidth memory (HBM) into the GPU's faster on-chip SRAM. This data movement constitutes latency and memory bottleneck for long sequences, which our work aims to alleviate.

\section{Methodology}
\label{sec:methodology}
Our proposed framework, SWAN as illustrated in Figure \ref{fig:swan_mechanism}, operates by rotating the key and value vectors into a low-dimensional subspace where information is maximally concentrated in the initial dimensions. This allows for effective pruning of the later, less important dimensions, leading to significant savings in both memory and computation without a reconstruction step. The core of our method is the offline creation of a powerful projection matrix.

\subsection{Construction of Projection Matrices}
The projection matrix is designed to find a basis that aligns the related components of the attention mechanism. Instead of learning separate bases for queries and keys (or values and outputs), we construct a unified basis for their joint distributions. This is a one-time, offline process.

\begin{figure}[t!]
    \centering
    \includegraphics[width=0.95\textwidth]{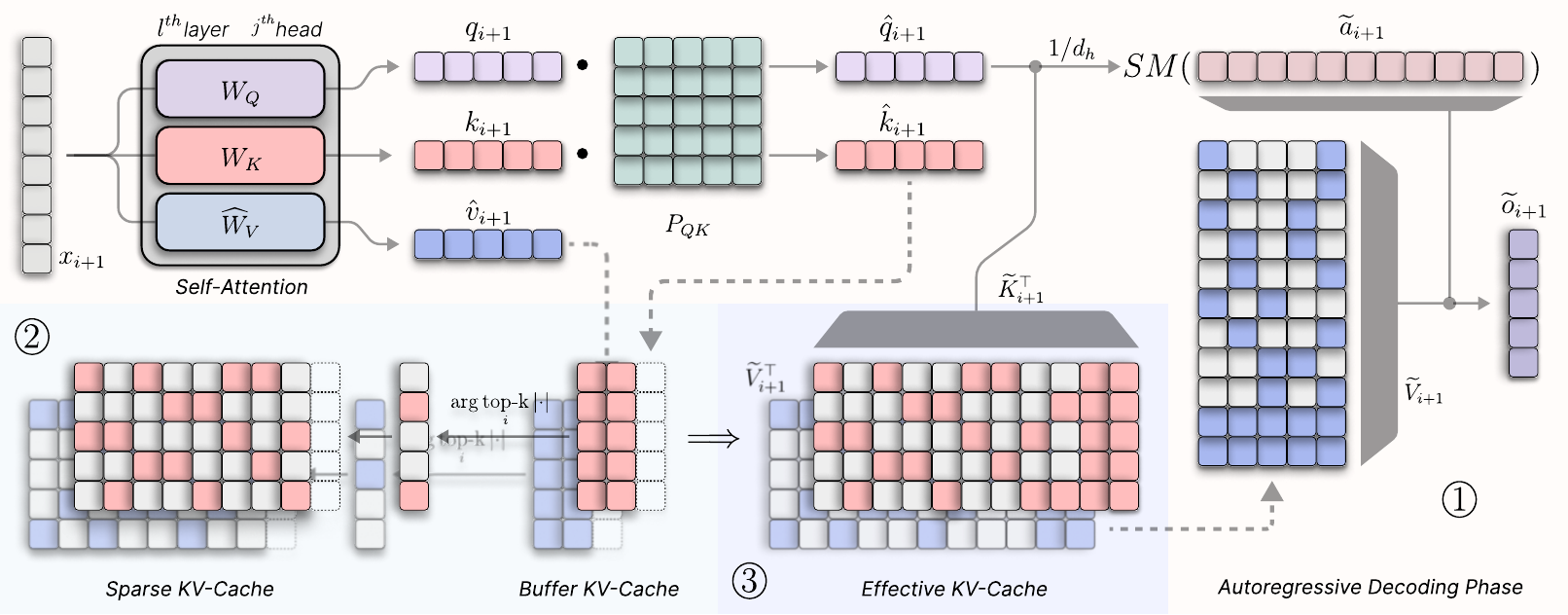}
    \caption{An illustration of the SWAN attention mechanism during a single autoregressive decoding step for token $i+1$. The process begins with the input $x_{i+1}$, where the query ($q_{i+1}$) and key ($k_{i+1}$) are projected at runtime by the orthogonal matrix $P_{QK}$ to produce their rotated counterparts, $\hat{q}_{i+1}$ and $\hat{k}_{i+1}$. The value vector is generated directly in the rotated space as $\hat{v}_{i+1}$ using the pre-modified weight matrix $\widehat{W}_V$. The core of our method is the hybrid KV-cache, composed of two parts: (2) a \textbf{Sparse KV-Cache} storing pruned historical vectors, and a small, dense \textbf{Buffer KV-Cache} for recent vectors. As new vectors ($\hat{k}_{i+1}, \hat{v}_{i+1}$) enter the buffer, the oldest buffer vector is pruned based on magnitude (`arg top-k') and moved to the sparse cache. The final attention output ($\tilde{o}_{i+1}$) is computed using the rotated query $\hat{q}_{i+1}$ and the (3) \textbf{Effective KV-Cache}, which is the combination of both sparse and buffer caches, thus avoiding any decompression overhead.}
    \label{fig:swan_mechanism}
\end{figure}

\subsubsection{Collecting Model's Activations}

A representative set of internal activations is collected from the target LLM. A subset of calibration dataset (e.g., BookCorpus) is processed through the model in a single forward pass. For each attention layer $l$, we extract the following: 1. Query ($Q^{(l)}$), and Key ($K^{(l)}$) matrices after the application of Rotary Positional Embeddings (RoPE) \citep{su2023roformerenhancedtransformerrotary}, as this reflects their state just before the attention score calculation, 2. Value ($V^{(l)}$) matrices and the output projection weights ($W_O^{(l)}$).

Our method is compatible with both Multi-Head Attention (MHA) and Grouped-Query Attention (GQA). In MHA, the number of query heads ($N_q$) equals the number of key-value heads ($N_{kv}$). In GQA, multiple query heads share a single KV-head ($N_q > N_{kv}$). To create a shared subspace that respects these architectures, we must align the interacting components.

For each KV-head, we group the corresponding $G = N_q / N_{kv}$ query heads. This is achieved by reshaping the query tensor. For a layer $l$, the query tensor $Q^{(l)} \in \mathbb{R}^{N_q \times n \times d_h}$, where $n$ is the number of tokens in the calibration sequence, is reshaped to $Q_{grouped}^{(l)} \in \mathbb{R}^{N_{kv} \times (n \cdot G) \times d_h}$.

A similar alignment is required for the output projection matrix, $W_O^{(l)} \in \mathbb{R}^{N_h d_h \times d}$. This matrix projects the concatenated outputs of all heads back to the model's dimension. Conceptually, we can slice this matrix into $N_h$ sub-matrices, $\{W_O^{(l,1)}, \dots, W_O^{(l,N_h)}\}$, where each sub-matrix $W_O^{(l,j)} \in \mathbb{R}^{d_h \times d}$ is responsible for projecting the output of its respective head $j$. To align these components with the value heads for our joint subspace construction, these sub-matrices are grouped in the same manner as the queries, resulting in $N_{kv}$ groups of output weight components, which we denote $W_{O, grouped}^{(l)}$.

\subsubsection{Forming Joint Subspaces and Deriving the Basis}
The central idea is to find a common basis for semantically related tensors. We create two joint matrices for each layer $l$ and each KV-head group $j$: one for Query-Key ($\mathbf{S}_{QK}^{(l,j)}$) and one for Value-Output ($\mathbf{S}_{VO}^{(l,j)}$). These are formed by concatenating the respective (grouped) tensors along the sequence dimension:
$
    \mathbf{S}_{QK}^{(l,j)} = \text{Concat}(Q_{grouped}^{(l,j)}, K^{(l,j)})
$,
$
    \mathbf{S}_{VO}^{(l,j)} = \text{Concat}(V^{(l,j)}, W_{O, grouped}^{(l,j)\top})
$.

We then apply Singular Value Decomposition (SVD) \citep{1102314} to these joint matrices to find their principal components. For a given joint matrix $\mathbf{S}$, SVD provides the decomposition $\mathbf{S} = U \Sigma V^T$. The columns of the right singular matrix $V$ form an orthonormal basis for the row space of $\mathbf{S}$, capturing the directions of greatest variance in descending order. This matrix $V$ becomes our projection matrix: 
$
    P_{QK}^{(l,j)} = \text{SVD}(\mathbf{S}_{QK}^{(l,j)})_V, \quad P_{VO}^{(l,j)} = \text{SVD}(\mathbf{S}_{VO}^{(l,j)})_V
$.
These projection matrices, $P_{QK}, P_{VO} \in \mathbb{R}^{d_h \times d_h}$, are computed once offline and stored. During inference, they are used to rotate the attention components into a space where pruning can be performed effectively.

\subsection{Applying Projection Matrices to Model Weights}
To minimize runtime overhead, the projection matrices can be intelligently ``absorbed'' into the model's original weights before inference begins. However, this optimization is only possible for certain components due to the operational order within the attention mechanism.

The projection for the Value and Output components, $P_{VO}$, can be pre-applied. The original value vector $v$ is produced by $v = x W_V$. The rotated vector is $\hat{v} = v P_{VO}$. We can absorb this rotation into a new weight matrix $\widehat{W}_{V}$ such that $\hat{v} = x \widehat{W}_{V}$. This gives 
$
    \widehat{W}_{V}^{(l,j)} = W_V^{(l,j)} P_{VO}^{(l,j)}
$. A corresponding transformation is applied to the output projection matrix $W_O$. First, we slice $W_O$ into sub-matrices $\{W_{O}^{(l,1)}, \dots, W_{O}^{(l,N_h)}\}$, one for each of the $N_h$ heads. Since the output of each head will now be in the rotated space ($\hat{o}_j = o_j P_{VO}$), we must modify each weight slice to account for this rotation. The new weight slice for each head $j$ is calculated by pre-multiplying with the transpose of the corresponding projection matrix:
$\widehat{W}_{O}^{(l,j)} = (P_{VO, expanded}^{(l,j)})^T W_O^{(l,j)}$. Here, $P_{VO, expanded}$ is the projection matrix repeated for each query head within a GQA group. These modified sub-matrices are then concatenated back together to form the single new output projection matrix for the layer, $\widehat{W}_{O}^{(l)}$. This absorption ensures that the value vectors are directly generated in the rotated space and the final projection accounts for this rotation, all with zero computational overhead during inference.

This pre-loading, however, is not feasible for the Query and Key weights. In modern LLMs, Rotary Positional Embeddings (RoPE) \citep{su2023roformerenhancedtransformerrotary} are applied to queries and keys \textit{after} their initial projection by $W_Q$ and $W_K$. Since our projection matrices were derived from RoPE-applied vectors, the projection must also occur after RoPE. The sequence of operations is $q \rightarrow \text{RoPE}(q) \rightarrow \text{RoPE}(q) P_{QK}$. Critically, RoPE is a dynamic, position-dependent linear transformation, and matrix multiplication is not commutative. It is not possible to find a static matrix $\widehat{W}_{Q}$ such that $\text{RoPE}(x \widehat{W}_{Q}) = \text{RoPE}(x W_Q) P_{QK}$ for all positions. Consequently, the projection $P_{QK}$ must be applied to the query and key vectors at runtime during each decoding step. This introduces a small computational cost. However, this fixed cost is quickly amortized by the substantial computational savings from the subsequent sparse attention calculation. As the sequence length grows, the savings in the dot-product computation, which scale with the sequence length, increasingly outweigh the initial projection cost. We will later provide a theoretical bound for the sequence length at which net computational savings begin. We formally prove that this entire rotation process is lossless in Appendix~\ref{appendix:proofs_lossless_rotation}, demonstrating that the only source of approximation error in our method is the subsequent dimension pruning.

\begin{algorithm}[!t]
\caption{SWAN Attention (for token i+1)}
\label{algo:swan_attention_algo}
\begin{algorithmic}[1]
\Require Current query $q_{i+1}$, key $k_{i+1}$, value $\hat{v}_{i+1}$ (already projected).
\Require KV-Cache: $K_{buffer}, K_{sparse}, V_{buffer}, V_{sparse}$.
\Require Projection matrix $P_{QK}$, buffer size $b$, top-k dims $k_{active}$.
\State $\hat{q}_{i+1} \leftarrow q_{i+1}P_{QK}$ \Comment{Project current query}
\State $\hat{k}_{i+1} \leftarrow k_{i+1}P_{QK}$ \Comment{Project current key}
\State Append $\hat{k}_{i+1}$ to $K_{buffer}$, $\hat{v}_{i+1}$ to $V_{buffer}$.
\If{size of $K_{buffer} > b$}
    \State $\hat{k}_{old} \leftarrow$ Pop oldest from $K_{buffer}$.
    \State $\hat{v}_{old} \leftarrow$ Pop oldest from $V_{buffer}$.
    \State $I_k \leftarrow \text{arg TopK}(|\hat{k}_{old}|, k_{active})$ \Comment{Find top-k indices for key}
    \State $k_{sparse} \leftarrow \text{Sparse}(\hat{k}_{old}, I_k)$ \Comment{Create sparse vector}
    \State Append $k_{sparse}$ to $K_{sparse}$.
    \State $I_v \leftarrow \text{arg TopK}(|\hat{v}_{old}|, k_{active})$ \Comment{Find top-k indices for value}
    \State $v_{sparse} \leftarrow \text{Sparse}(\hat{v}_{old}, I_v)$
    \State Append $v_{sparse}$ to $V_{sparse}$.
\EndIf
\State $K_{cache} \leftarrow \text{Concat}(K_{sparse}, K_{buffer})$
\State $V_{cache} \leftarrow \text{Concat}(V_{sparse}, V_{buffer})$
\State $S \leftarrow \hat{q}_{i+1}K_{cache}^T$ \Comment{Sparse-dense \& dense-dense mat-vec products}
\State $o_{i+1} \leftarrow \text{softmax}(S/\sqrt{d_h}) V_{cache}$ \Comment{Compute attention output}
\State \Return $o_{i+1}$
\end{algorithmic}
\end{algorithm}

\subsection{Inference-Time Pruning and Attention Computation}
With the projection matrices established and partially absorbed, our method modifies the standard autoregressive decoding loop. The core innovation is a hybrid KV-cache strategy that combines a large, sparse cache for historical tokens with a small, dense buffer for recent ones. Consistent with findings in prior work \citep{https://doi.org/10.13140/rg.2.2.28167.37282, kang2024gearefficientkvcache, kim2024lexicoextremekvcache}, we observe that maintaining this small buffer of recent tokens in their dense format is crucial for preserving model performance, a finding we empirically validate in our results section.

At each generation step $i+1$, the new query $q_{i+1}$ and key $k_{i+1}$ are projected into the rotated space at runtime using $P_{QK}$. The new value vector $\hat{v}_{i+1}$ is generated directly in the rotated space, thanks to the absorbed projection in $\widehat{W}_V$. These new, dense projected vectors are temporarily stored in the fixed-size buffer. When the buffer exceeds its capacity, the oldest dense vectors are evicted, pruned, and added to the main sparse cache.

The pruning process is based on magnitude. For an evicted vector (e.g., $\hat{k}_{old}$), we identify the top-$k$ dimensions with the highest absolute values, where $k$ is a tunable hyperparameter controlling the compression ratio. All other dimensions are discarded. The vector is then converted to a sparse representation (e.g., storing only the indices and values of the top-$k$ elements) and appended to the historical sparse cache. This procedure is summarized as Algorithm \ref{algo:swan_attention_algo}.

This approach yields significant computational savings. The attention score calculation involves multiplying the dense query vector with the hybrid key cache. The operation on the sparse portion of the cache is a sparse-dense matrix-vector product, which is computationally much cheaper than a standard dense-dense operation.

Furthermore, to achieve higher compression rates with minimal performance loss, the non-zero values stored in the sparse vectors can be quantized to 8-bit float. This allows us to retain a larger number of components (a higher $k_{active}$) for the same memory budget. Crucially, unlike token eviction strategies that cause complete information loss for discarded tokens, our method ensures that some information from every token is retained in the cache, preserving a more complete history of the sequence context.

\section{Complexity Analysis}
\label{sec:complexity_analysis}

\subsection{Space Complexity}
The memory savings of our method come from the sparse representation of historical key and value vectors. The sparse vectors are stored using the Compressed Sparse Row (CSR) array format. A single vector is stored by keeping only its $k_{active}$ most significant components as a tuple of values and their corresponding indices, plus a constant 2-byte offset array.

Assuming a typical head dimension $d_h=128$, the indices can be stored using 8-bit integers (\texttt{int8}), as they range from 0 to 127. The values are typically stored as 16-bit floating-point numbers (\texttt{float16}). Thus, the memory required for one sparse vector is:
\begin{equation}
    M_{sparse} = k_{active} \cdot (\text{sizeof}(\texttt{float16}) + \text{sizeof}(\texttt{int8})) + 2 = 3k_{active} + 2 \text{ bytes}
\end{equation}
In contrast, a standard dense vector requires $M_{dense} = d_h \cdot \text{sizeof}(\texttt{float16}) = 128 \cdot 2 = 256$ bytes. The compression ratio for the sparse portion of the cache is therefore $\frac{3k_{active} + 2}{256}$.

For aggressive compression, the values can be further quantized to 8-bit floats (\texttt{float8}). This reduces the memory per vector to $M_{sparse,8bit} = k_{active} \cdot (1 + 1) + 2 = 2k_{active} + 2$ bytes. This allows us to retain more components for the same memory budget, which is critical for performance at high compression ratios, as it preserves more information per token compared to simply reducing $k_{active}$.

\subsection{Computational Complexity}
\label{sec:complexity_analysis_main}
We theoretically analyze the computational complexity by comparing the number of floating-point operations (FLOPs) required for the full attention computation within a single head at a decoding step for a sequence of length $L$. Our analysis reveals that SWAN becomes computationally more efficient than standard attention once the sequence length surpasses a predictable threshold. This break-even point is defined in terms of the head dimension ($d_h$), the dense buffer size ($b$), and the number of active dimensions retained after pruning ($k_{active}$), and is reached when the sequence length $L$ satisfies the following condition:


\begin{equation}
L > \frac{d_h^2}{d_h - k_{active}} + b
\label{eq:break_even}
\end{equation}

Equation~\ref{eq:break_even} provides a powerful and practical guide for deploying SWAN. It clearly defines the trade-off between the initial overhead and the accumulated savings. The term $d_h^2$ represents the fixed, one-time computational cost of the runtime projections at each step. The denominator, $d_h - k_{active}$, represents the per-token computational savings achieved by operating on sparse vectors. Finally, $b$ accounts for the portion of the computation that remains dense due to the buffer.

This relationship demonstrates a key advantage of our method: the more aggressively we prune the vectors (i.e., the smaller the $k_{active}$), the larger the per-token savings become, and the shorter the sequence length required to ``pay off'' the initial projection overhead. This formula allows operators to precisely determine the context length at which SWAN not only saves memory but also begins to accelerate inference, making it a highly predictable and advantageous solution for long-context applications. The full mathematical derivation for this result is provided in Appendix~\ref{app:complexity_derivation}.

\begin{figure}[t!]
    \centering
    \begin{minipage}{0.8\linewidth} 
        \centering
        \begin{subfigure}{0.4\linewidth}
            \centering
            \includegraphics[width=\linewidth]{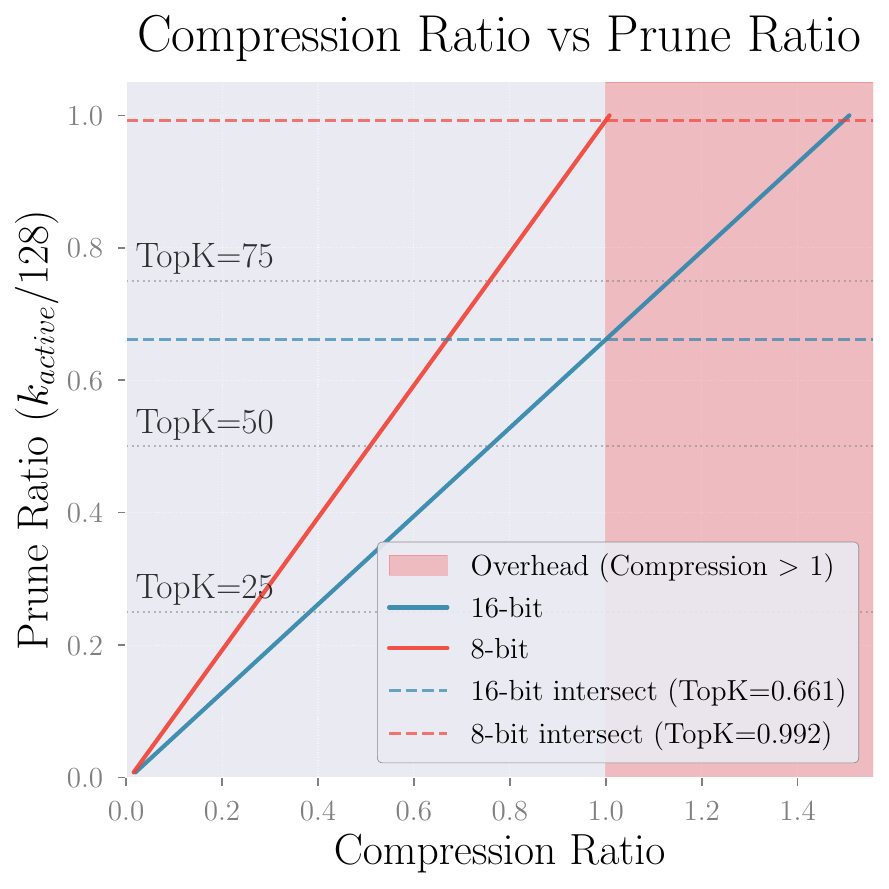}
            \caption{Compression-Pruning tradeoff}
            \label{fig:compression_tradeoff}
        \end{subfigure}\hfill
        \begin{subfigure}{0.52\linewidth}
            \centering
            \includegraphics[width=\linewidth]{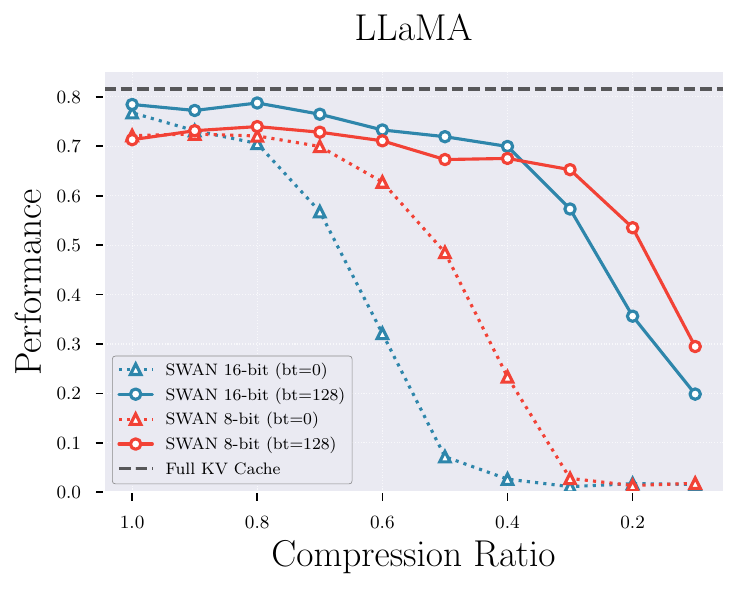}
            \caption{Llama-3.1-8B GSM8K results}
            \label{fig:gsm8k_results}
        \end{subfigure}
    \end{minipage}
    \caption{(a) Relationship between pruning ratio (dimensions retained) and effective memory compression. The shaded region indicates where the sparse representation is larger than the dense original. For 16-bit values, savings begin only when the retention ratio is below 0.66, this threshold is significantly lower when using 8-bit quantized values and is almost one-to-one. (b) Performance of SWAN variants of \texttt{Llama-3.1-8B-Instruct} on GSM8K reasoning benchmark. The buffered SWAN variants (`bt=128') demonstrate strong resilience, significantly outperforming the zero-buffer versions}
\end{figure}

\section{Experimental Results}
\label{sec:results}

\subsection{The Compression-Pruning Trade-off}

In our method, pruning dimensions does not linearly translate to memory savings due to the overhead of storing sparse indices. Figure~\ref{fig:compression_tradeoff} illustrates this critical relationship. For standard 16-bit values, we must prune over 34\% of dimensions (i.e., retain fewer than 66\%) just to break even on memory usage. Fortunately, our empirical results (see Appendix~\ref{app:perf_vs_pruning}) show that SWAN is highly resilient to such pruning, maintaining strong performance even when retaining only half of the original dimensions. This confirms that achieving significant memory savings is practical. Figure~\ref{fig:compression_tradeoff} also reveals that quantizing values to 8-bit floats makes the compression far more efficient. This insight frames a central question for our experiments: for a given memory budget, is it better to keep fewer, high-precision 16-bit dimensions, or a larger number of less-precise 8-bit ones? The following sections are designed to answer this.

\subsection{Performance on Reasoning Task}

Multi-step reasoning tasks like GSM8K serve as a powerful stress-test for KV-cache compression, as they are exceptionally unforgiving to information loss that can break a model's chain of thought. We use this benchmark to demonstrate SWAN's resilience under pressure on \texttt{Llama-3.1-8B-Instruct} \citep{grattafiori2024llama3herdmodels} model. The full details of the benchmarks and the corresponding settings in present and subsequent sub-sections are provided in Appendix~\ref{appendix:benchmarks}.

The results, presented in Figure~\ref{fig:gsm8k_results}, highlight the critical role of our hybrid cache design. The zero-buffer variants (`bt=0') suffer a catastrophic performance collapse, confirming that for complex reasoning, a high-fidelity buffer of recent context is non-negotiable. With a 128-token buffer, however, SWAN is remarkably robust. The 16-bit buffered variant maintains near-baseline performance even at a 50\% compression ratio (0.5 on the x-axis), proving that our method can preserve complex reasoning capabilities with substantial memory savings.

This task also reveals a fascinating trade-off between precision and the number of retained dimensions. Initially, the 16-bit SWAN variant has a slight edge, where its numerical precision is beneficial. However, as we push compression below a 40\% ratio, a clear crossover occurs: the 8-bit variant becomes superior. This demonstrates that in extreme memory-constrained scenarios, the information lost from discarding too many dimensions (the 16-bit case) is more damaging than the precision lost from quantization. At this point, retaining a broader, albeit less precise, set of dimensions provides a decisive advantage for maintaining the logical chain.

\subsection{Performance on Standard NLP Benchmarks}
\label{sec:nlp_benchmarks}

\begin{figure*}[t!]
    \centering
    \includegraphics[width=\textwidth]{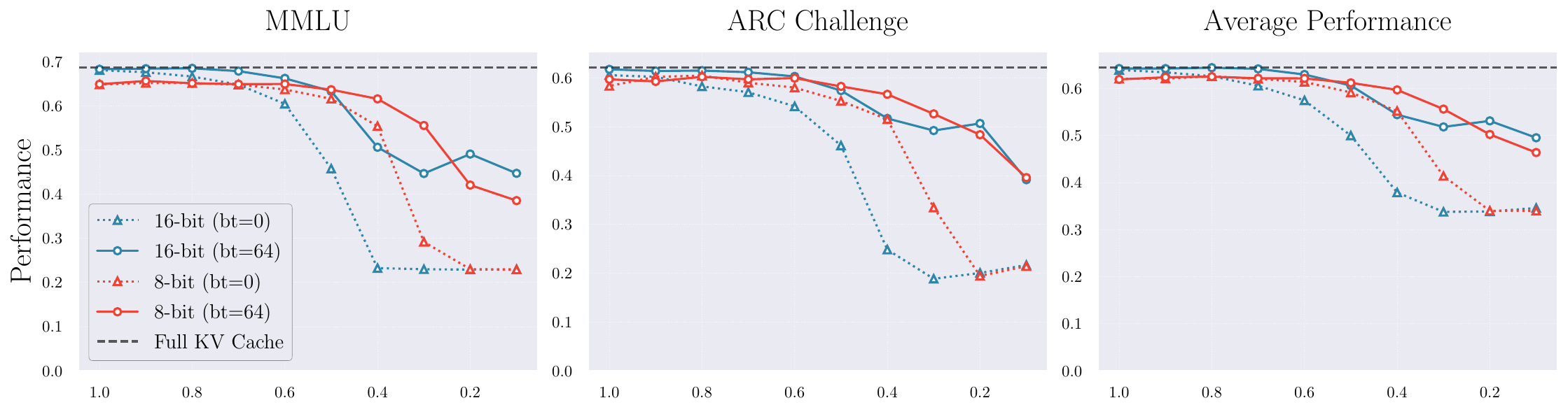}
    \includegraphics[width=\textwidth]{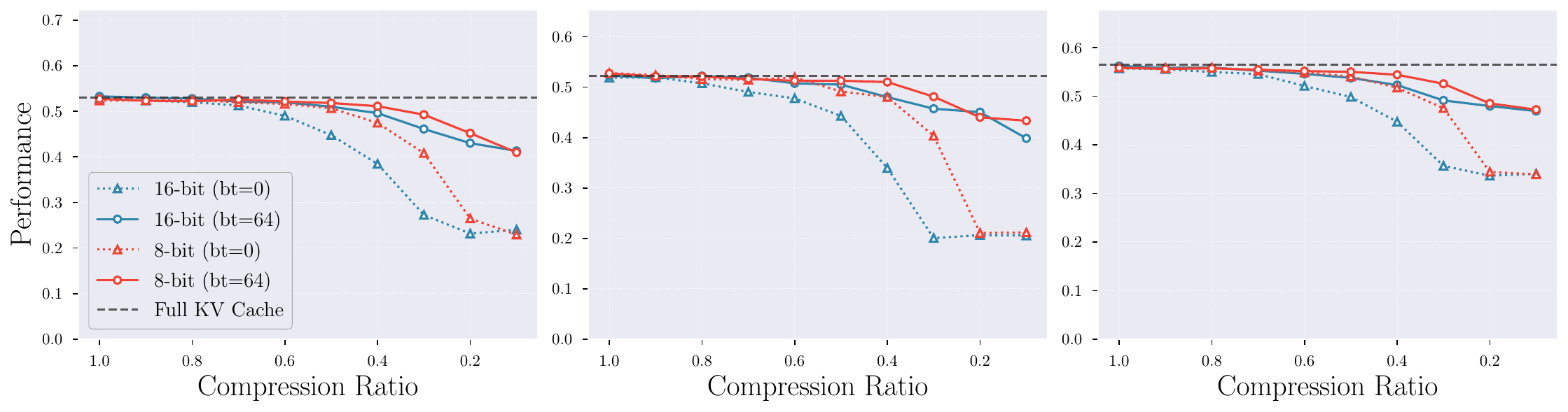}
    \caption{Performance on key NLP benchmarks for \texttt{Llama-3.1-8B-Instruct} (top) and \texttt{OLMoE-1B-7B-Instruct} (bottom). The buffered SWAN (`bt=64') maintains high performance even at significant compression ratios. Note the consistently smaller performance drop on the sparser OLMoE model, highlighting SWAN's ability to exploit the inherent sparsity in model architectures.}
    \label{fig:nlp_summary}
\end{figure*}

To demonstrate SWAN's generalizability, we evaluate it on a range of NLP benchmarks using two distinct architectures: the dense Grouped-Query Attention (GQA) of \texttt{Llama-3.1-8B-Instruct} \citep{grattafiori2024llama3herdmodels} and the inherently sparser Multi-Head Attention (MHA) of \texttt{OLMoE-1B-7B-Instruct} \citep{muennighoff2025olmoeopenmixtureofexpertslanguage}. This dual evaluation allows us to assess our method's performance on different attention structures.

The results, shown in Figure~\ref{fig:nlp_summary}, reveal three key findings. First, the 64-token dense buffer is critical for performance across all tasks. The buffered variants maintain high accuracy even with 50-60\% memory savings, while the zero-buffer versions degrade sharply. Second, for knowledge-intensive tasks like MMLU and ARC Challenge, the 8-bit variant is particularly effective under high compression. Similar to GSM8K task, its ability to retain more, albeit less precise, dimensions proves advantageous for factual recall.

Most importantly, our results reveal a powerful synergy between SWAN and a model's inherent architecture. The performance degradation on the sparser MHA-based OLMoE is consistently lower than on the GQA-based Llama across all tasks. This effect is most dramatically illustrated by the WikiText perplexity benchmark (see Appendix~\ref{appendix:nlp_details}), where the performance drop on OLMoE is three times less severe. This is a crucial finding: SWAN is not just imposing sparsity but is effectively leveraging the natural, learned structure of the model's attention mechanism.  It demonstrates that SWAN directly attacks and leverages the \textit{inherent sparsity} of a model's attention mechanism - a capability lacking in previous works. This explains its remarkable performance even without any reconstruction step and makes it an elegant solution for diverse model designs.

\subsection{Long-Context Evaluation}
\label{sec:long_context} 

\begin{figure*}[t!]
    \centering
    \includegraphics[width=\textwidth]{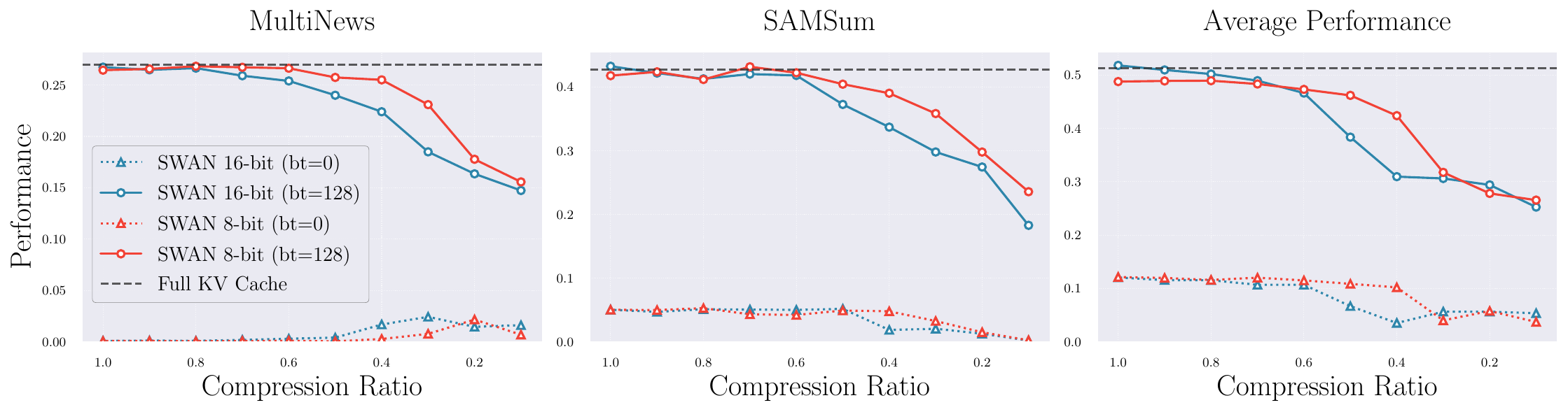}
    \caption{SWAN's performance on the LongBench suite. The figure highlights the performance across two summarization tasks, Multi-News and SAMSum, as well as the average performance across the MultiNews, LCC, SAMSum, Multi-News, and TREC tasks.The 128-token buffer (`bt=128') is critical, preventing the catastrophic failure seen in the zero-buffer versions. Note the graceful degradation of the buffered variants and the strong performance of the 8-bit version on summarization, even under aggressive compression.}
    \label{fig:longbench_summary}
\end{figure*}

Long-context tasks are the ultimate stress-test for a KV-cache compression method, as even minor, cumulative information loss can derail a model's understanding over thousands of tokens. We evaluate SWAN on the demanding LongBench suite \citep{bai2024longbenchbilingualmultitaskbenchmark} to demonstrate its capability in this challenging regime.

The results in Figure~\ref{fig:longbench_summary} are definitive: SWAN's hybrid cache design is essential for long-context performance. The zero-buffer variants (`bt=0') suffer a complete performance collapse across all tasks. This is because completely sparsifying the cache from the start prevents the model from forming a coherent understanding of the initial context. In contrast, the 128-token dense buffer acts as a high-fidelity `working memory', allowing the model to process local information accurately, while the growing sparse cache preserves the long-range dependencies of the document's `gist'.

With this crucial buffer in place, SWAN's robustness is remarkable. The average performance degrades gracefully, remaining highly competitive with the uncompressed baseline even with 50-60\% memory savings. On summarization tasks like Multi-News and SAMSum, the 8-bit buffered variant is particularly effective, often outperforming its 16-bit counterpart at high compression ratios. This again aligns with our previous findings that retaining a broader number of components, even at lower precision is more beneficial than keeping fewer, high-precision ones. This resilience is not limited to summarization, as our method shows similar strength on diverse tasks like code completion and passage retrieval (see Appendix~\ref{appendix:longbench_details}).

\section{Conclusion}
\label{sec:conclusion}

In this work, we introduced SWAN, a fine-tuning-free framework that fundamentally rethinks KV-cache compression by eliminating the costly, decompression step common to prior low-rank methods. Our approach performs attention directly on a pruned, sparse cache, simultaneously reducing both memory and compute load. Our extensive experiments demonstrated that SWAN, augmented with a small dense buffer, maintains robust performance even at 50-60\% memory savings, and uniquely benefits from sparser attention architectures. Its key advantage lies in its runtime-tunable design, which offers a level of operational flexibility absent in methods with fixed offline configurations. As a practical solution for efficient, long-context inference, SWAN's full potential will be further realized with advancements in hardware and software optimized for sparse computation.

\newpage

\bibliography{main}

@misc{vaswani2023attentionneed,
      title={Attention Is All You Need}, 
      author={Ashish Vaswani and Noam Shazeer and Niki Parmar and Jakob Uszkoreit and Llion Jones and Aidan N. Gomez and Lukasz Kaiser and Illia Polosukhin},
      year={2023},
      eprint={1706.03762},
      archivePrefix={arXiv},
      primaryClass={cs.CL},
      url={https://arxiv.org/abs/1706.03762}, 
}

@misc{ainslie2023gqatraininggeneralizedmultiquery,
      title={GQA: Training Generalized Multi-Query Transformer Models from Multi-Head Checkpoints}, 
      author={Joshua Ainslie and James Lee-Thorp and Michiel de Jong and Yury Zemlyanskiy and Federico Lebrón and Sumit Sanghai},
      year={2023},
      eprint={2305.13245},
      archivePrefix={arXiv},
      primaryClass={cs.CL},
      url={https://arxiv.org/abs/2305.13245}, 
}

@misc{shazeer2019fasttransformerdecodingwritehead,
      title={Fast Transformer Decoding: One Write-Head is All You Need}, 
      author={Noam Shazeer},
      year={2019},
      eprint={1911.02150},
      archivePrefix={arXiv},
      primaryClass={cs.NE},
      url={https://arxiv.org/abs/1911.02150}, 
}

@misc{hooper2025kvquant10millioncontext,
      title={KVQuant: Towards 10 Million Context Length LLM Inference with KV Cache Quantization}, 
      author={Coleman Hooper and Sehoon Kim and Hiva Mohammadzadeh and Michael W. Mahoney and Yakun Sophia Shao and Kurt Keutzer and Amir Gholami},
      year={2025},
      eprint={2401.18079},
      archivePrefix={arXiv},
      primaryClass={cs.LG},
      url={https://arxiv.org/abs/2401.18079}, 
}

@misc{zhang2023h2oheavyhitteroracleefficient,
      title={H$_2$O: Heavy-Hitter Oracle for Efficient Generative Inference of Large Language Models}, 
      author={Zhenyu Zhang and Ying Sheng and Tianyi Zhou and Tianlong Chen and Lianmin Zheng and Ruisi Cai and Zhao Song and Yuandong Tian and Christopher Ré and Clark Barrett and Zhangyang Wang and Beidi Chen},
      year={2023},
      eprint={2306.14048},
      archivePrefix={arXiv},
      primaryClass={cs.LG},
      url={https://arxiv.org/abs/2306.14048}, 
}

@misc{saxena2024eigenattentionattentionlowrank,
      title={Eigen Attention: Attention in Low-Rank Space for KV Cache Compression}, 
      author={Utkarsh Saxena and Gobinda Saha and Sakshi Choudhary and Kaushik Roy},
      year={2024},
      eprint={2408.05646},
      archivePrefix={arXiv},
      primaryClass={cs.LG},
      url={https://arxiv.org/abs/2408.05646}, 
}

@misc{singhania2024lokilowrankkeysefficient,
      title={Loki: Low-rank Keys for Efficient Sparse Attention}, 
      author={Prajwal Singhania and Siddharth Singh and Shwai He and Soheil Feizi and Abhinav Bhatele},
      year={2024},
      eprint={2406.02542},
      archivePrefix={arXiv},
      primaryClass={cs.LG},
      url={https://arxiv.org/abs/2406.02542}, 
}

@misc{kwon2023efficientmemorymanagementlarge,
      title={Efficient Memory Management for Large Language Model Serving with PagedAttention}, 
      author={Woosuk Kwon and Zhuohan Li and Siyuan Zhuang and Ying Sheng and Lianmin Zheng and Cody Hao Yu and Joseph E. Gonzalez and Hao Zhang and Ion Stoica},
      year={2023},
      eprint={2309.06180},
      archivePrefix={arXiv},
      primaryClass={cs.LG},
      url={https://arxiv.org/abs/2309.06180}, 
}

@misc{xiao2024efficientstreaminglanguagemodels,
      title={Efficient Streaming Language Models with Attention Sinks}, 
      author={Guangxuan Xiao and Yuandong Tian and Beidi Chen and Song Han and Mike Lewis},
      year={2024},
      eprint={2309.17453},
      archivePrefix={arXiv},
      primaryClass={cs.CL},
      url={https://arxiv.org/abs/2309.17453}, 
}

@article{https://doi.org/10.13140/rg.2.2.28167.37282,
  doi = {10.13140/RG.2.2.28167.37282},
  
  url = {https://rgdoi.net/10.13140/RG.2.2.28167.37282},
  
  author = {{Zirui Liu} and {Jiayi Yuan} and {Hongye Jin} and {Shaochen Zhong} and {Zhaozhuo Xu} and Braverman, Vladimir and {Beidi Chen} and Hu, Xia},
  
  language = {en},
  
  title = {KIVI : Plug-and-play 2bit KV Cache Quantization with Streaming Asymmetric Quantization},
  
  publisher = {Unpublished},
  
  year = {2023}
}

@misc{ribar2024sparqattentionbandwidthefficientllm,
      title={SparQ Attention: Bandwidth-Efficient LLM Inference}, 
      author={Luka Ribar and Ivan Chelombiev and Luke Hudlass-Galley and Charlie Blake and Carlo Luschi and Douglas Orr},
      year={2024},
      eprint={2312.04985},
      archivePrefix={arXiv},
      primaryClass={cs.LG},
      url={https://arxiv.org/abs/2312.04985}, 
}

@misc{kim2024lexicoextremekvcache,
      title={Lexico: Extreme KV Cache Compression via Sparse Coding over Universal Dictionaries}, 
      author={Junhyuck Kim and Jongho Park and Jaewoong Cho and Dimitris Papailiopoulos},
      year={2024},
      eprint={2412.08890},
      archivePrefix={arXiv},
      primaryClass={cs.LG},
      url={https://arxiv.org/abs/2412.08890}, 
}

@misc{s2025aquaattentionquerymagnitudes,
      title={AQUA: Attention via QUery mAgnitudes for Memory and Compute Efficient Inference in LLMs}, 
      author={Santhosh G S and Saurav Prakash and Balaraman Ravindran},
      year={2025},
      eprint={2509.11155},
      archivePrefix={arXiv},
      primaryClass={cs.LG},
      url={https://arxiv.org/abs/2509.11155}, 
}

@misc{kang2024gearefficientkvcache,
      title={GEAR: An Efficient KV Cache Compression Recipe for Near-Lossless Generative Inference of LLM}, 
      author={Hao Kang and Qingru Zhang and Souvik Kundu and Geonhwa Jeong and Zaoxing Liu and Tushar Krishna and Tuo Zhao},
      year={2024},
      eprint={2403.05527},
      archivePrefix={arXiv},
      primaryClass={cs.LG},
      url={https://arxiv.org/abs/2403.05527}, 
}

@misc{grattafiori2024llama3herdmodels,
      title={The Llama 3 Herd of Models}, 
      author={Aaron Grattafiori and Abhimanyu Dubey and Abhinav Jauhri and Abhinav Pandey and Abhishek Kadian and Ahmad Al-Dahle and Aiesha Letman and Akhil Mathur and Alan Schelten and Alex Vaughan and Amy Yang and Angela Fan and Anirudh Goyal and Anthony Hartshorn and Aobo Yang and Archi Mitra and Archie Sravankumar and Artem Korenev and Arthur Hinsvark and Arun Rao and Aston Zhang and Aurelien Rodriguez and Austen Gregerson and Ava Spataru and Baptiste Roziere and Bethany Biron and Binh Tang and Bobbie Chern and Charlotte Caucheteux and Chaya Nayak and Chloe Bi and Chris Marra and Chris McConnell and Christian Keller and Christophe Touret and Chunyang Wu and Corinne Wong and Cristian Canton Ferrer and Cyrus Nikolaidis and Damien Allonsius and Daniel Song and Danielle Pintz and Danny Livshits and Danny Wyatt and David Esiobu and Dhruv Choudhary and Dhruv Mahajan and Diego Garcia-Olano and Diego Perino and Dieuwke Hupkes and Egor Lakomkin and Ehab AlBadawy and Elina Lobanova and Emily Dinan and Eric Michael Smith and Filip Radenovic and Francisco Guzmán and Frank Zhang and Gabriel Synnaeve and Gabrielle Lee and Georgia Lewis Anderson and Govind Thattai and Graeme Nail and Gregoire Mialon and Guan Pang and Guillem Cucurell and Hailey Nguyen and Hannah Korevaar and Hu Xu and Hugo Touvron and Iliyan Zarov and Imanol Arrieta Ibarra and Isabel Kloumann and Ishan Misra and Ivan Evtimov and Jack Zhang and Jade Copet and Jaewon Lee and Jan Geffert and Jana Vranes and Jason Park and Jay Mahadeokar and Jeet Shah and Jelmer van der Linde and Jennifer Billock and Jenny Hong and Jenya Lee and Jeremy Fu and Jianfeng Chi and Jianyu Huang and Jiawen Liu and Jie Wang and Jiecao Yu and Joanna Bitton and Joe Spisak and Jongsoo Park and Joseph Rocca and Joshua Johnstun and Joshua Saxe and Junteng Jia and Kalyan Vasuden Alwala and Karthik Prasad and Kartikeya Upasani and Kate Plawiak and Ke Li and Kenneth Heafield and Kevin Stone and Khalid El-Arini and Krithika Iyer and Kshitiz Malik and Kuenley Chiu and Kunal Bhalla and Kushal Lakhotia and Lauren Rantala-Yeary and Laurens van der Maaten and Lawrence Chen and Liang Tan and Liz Jenkins and Louis Martin and Lovish Madaan and Lubo Malo and Lukas Blecher and Lukas Landzaat and Luke de Oliveira and Madeline Muzzi and Mahesh Pasupuleti and Mannat Singh and Manohar Paluri and Marcin Kardas and Maria Tsimpoukelli and Mathew Oldham and Mathieu Rita and Maya Pavlova and Melanie Kambadur and Mike Lewis and Min Si and Mitesh Kumar Singh and Mona Hassan and Naman Goyal and Narjes Torabi and Nikolay Bashlykov and Nikolay Bogoychev and Niladri Chatterji and Ning Zhang and Olivier Duchenne and Onur Çelebi and Patrick Alrassy and Pengchuan Zhang and Pengwei Li and Petar Vasic and Peter Weng and Prajjwal Bhargava and Pratik Dubal and Praveen Krishnan and Punit Singh Koura and Puxin Xu and Qing He and Qingxiao Dong and Ragavan Srinivasan and Raj Ganapathy and Ramon Calderer and Ricardo Silveira Cabral and Robert Stojnic and Roberta Raileanu and Rohan Maheswari and Rohit Girdhar and Rohit Patel and Romain Sauvestre and Ronnie Polidoro and Roshan Sumbaly and Ross Taylor and Ruan Silva and Rui Hou and Rui Wang and Saghar Hosseini and Sahana Chennabasappa and Sanjay Singh and Sean Bell and Seohyun Sonia Kim and Sergey Edunov and Shaoliang Nie and Sharan Narang and Sharath Raparthy and Sheng Shen and Shengye Wan and Shruti Bhosale and Shun Zhang and Simon Vandenhende and Soumya Batra and Spencer Whitman and Sten Sootla and Stephane Collot and Suchin Gururangan and Sydney Borodinsky and Tamar Herman and Tara Fowler and Tarek Sheasha and Thomas Georgiou and Thomas Scialom and Tobias Speckbacher and Todor Mihaylov and Tong Xiao and Ujjwal Karn and Vedanuj Goswami and Vibhor Gupta and Vignesh Ramanathan and Viktor Kerkez and Vincent Gonguet and Virginie Do and Vish Vogeti and Vítor Albiero and Vladan Petrovic and Weiwei Chu and Wenhan Xiong and Wenyin Fu and Whitney Meers and Xavier Martinet and Xiaodong Wang and Xiaofang Wang and Xiaoqing Ellen Tan and Xide Xia and Xinfeng Xie and Xuchao Jia and Xuewei Wang and Yaelle Goldschlag and Yashesh Gaur and Yasmine Babaei and Yi Wen and Yiwen Song and Yuchen Zhang and Yue Li and Yuning Mao and Zacharie Delpierre Coudert and Zheng Yan and Zhengxing Chen and Zoe Papakipos and Aaditya Singh and Aayushi Srivastava and Abha Jain and Adam Kelsey and Adam Shajnfeld and Adithya Gangidi and Adolfo Victoria and Ahuva Goldstand and Ajay Menon and Ajay Sharma and Alex Boesenberg and Alexei Baevski and Allie Feinstein and Amanda Kallet and Amit Sangani and Amos Teo and Anam Yunus and Andrei Lupu and Andres Alvarado and Andrew Caples and Andrew Gu and Andrew Ho and Andrew Poulton and Andrew Ryan and Ankit Ramchandani and Annie Dong and Annie Franco and Anuj Goyal and Aparajita Saraf and Arkabandhu Chowdhury and Ashley Gabriel and Ashwin Bharambe and Assaf Eisenman and Azadeh Yazdan and Beau James and Ben Maurer and Benjamin Leonhardi and Bernie Huang and Beth Loyd and Beto De Paola and Bhargavi Paranjape and Bing Liu and Bo Wu and Boyu Ni and Braden Hancock and Bram Wasti and Brandon Spence and Brani Stojkovic and Brian Gamido and Britt Montalvo and Carl Parker and Carly Burton and Catalina Mejia and Ce Liu and Changhan Wang and Changkyu Kim and Chao Zhou and Chester Hu and Ching-Hsiang Chu and Chris Cai and Chris Tindal and Christoph Feichtenhofer and Cynthia Gao and Damon Civin and Dana Beaty and Daniel Kreymer and Daniel Li and David Adkins and David Xu and Davide Testuggine and Delia David and Devi Parikh and Diana Liskovich and Didem Foss and Dingkang Wang and Duc Le and Dustin Holland and Edward Dowling and Eissa Jamil and Elaine Montgomery and Eleonora Presani and Emily Hahn and Emily Wood and Eric-Tuan Le and Erik Brinkman and Esteban Arcaute and Evan Dunbar and Evan Smothers and Fei Sun and Felix Kreuk and Feng Tian and Filippos Kokkinos and Firat Ozgenel and Francesco Caggioni and Frank Kanayet and Frank Seide and Gabriela Medina Florez and Gabriella Schwarz and Gada Badeer and Georgia Swee and Gil Halpern and Grant Herman and Grigory Sizov and Guangyi and Zhang and Guna Lakshminarayanan and Hakan Inan and Hamid Shojanazeri and Han Zou and Hannah Wang and Hanwen Zha and Haroun Habeeb and Harrison Rudolph and Helen Suk and Henry Aspegren and Hunter Goldman and Hongyuan Zhan and Ibrahim Damlaj and Igor Molybog and Igor Tufanov and Ilias Leontiadis and Irina-Elena Veliche and Itai Gat and Jake Weissman and James Geboski and James Kohli and Janice Lam and Japhet Asher and Jean-Baptiste Gaya and Jeff Marcus and Jeff Tang and Jennifer Chan and Jenny Zhen and Jeremy Reizenstein and Jeremy Teboul and Jessica Zhong and Jian Jin and Jingyi Yang and Joe Cummings and Jon Carvill and Jon Shepard and Jonathan McPhie and Jonathan Torres and Josh Ginsburg and Junjie Wang and Kai Wu and Kam Hou U and Karan Saxena and Kartikay Khandelwal and Katayoun Zand and Kathy Matosich and Kaushik Veeraraghavan and Kelly Michelena and Keqian Li and Kiran Jagadeesh and Kun Huang and Kunal Chawla and Kyle Huang and Lailin Chen and Lakshya Garg and Lavender A and Leandro Silva and Lee Bell and Lei Zhang and Liangpeng Guo and Licheng Yu and Liron Moshkovich and Luca Wehrstedt and Madian Khabsa and Manav Avalani and Manish Bhatt and Martynas Mankus and Matan Hasson and Matthew Lennie and Matthias Reso and Maxim Groshev and Maxim Naumov and Maya Lathi and Meghan Keneally and Miao Liu and Michael L. Seltzer and Michal Valko and Michelle Restrepo and Mihir Patel and Mik Vyatskov and Mikayel Samvelyan and Mike Clark and Mike Macey and Mike Wang and Miquel Jubert Hermoso and Mo Metanat and Mohammad Rastegari and Munish Bansal and Nandhini Santhanam and Natascha Parks and Natasha White and Navyata Bawa and Nayan Singhal and Nick Egebo and Nicolas Usunier and Nikhil Mehta and Nikolay Pavlovich Laptev and Ning Dong and Norman Cheng and Oleg Chernoguz and Olivia Hart and Omkar Salpekar and Ozlem Kalinli and Parkin Kent and Parth Parekh and Paul Saab and Pavan Balaji and Pedro Rittner and Philip Bontrager and Pierre Roux and Piotr Dollar and Polina Zvyagina and Prashant Ratanchandani and Pritish Yuvraj and Qian Liang and Rachad Alao and Rachel Rodriguez and Rafi Ayub and Raghotham Murthy and Raghu Nayani and Rahul Mitra and Rangaprabhu Parthasarathy and Raymond Li and Rebekkah Hogan and Robin Battey and Rocky Wang and Russ Howes and Ruty Rinott and Sachin Mehta and Sachin Siby and Sai Jayesh Bondu and Samyak Datta and Sara Chugh and Sara Hunt and Sargun Dhillon and Sasha Sidorov and Satadru Pan and Saurabh Mahajan and Saurabh Verma and Seiji Yamamoto and Sharadh Ramaswamy and Shaun Lindsay and Shaun Lindsay and Sheng Feng and Shenghao Lin and Shengxin Cindy Zha and Shishir Patil and Shiva Shankar and Shuqiang Zhang and Shuqiang Zhang and Sinong Wang and Sneha Agarwal and Soji Sajuyigbe and Soumith Chintala and Stephanie Max and Stephen Chen and Steve Kehoe and Steve Satterfield and Sudarshan Govindaprasad and Sumit Gupta and Summer Deng and Sungmin Cho and Sunny Virk and Suraj Subramanian and Sy Choudhury and Sydney Goldman and Tal Remez and Tamar Glaser and Tamara Best and Thilo Koehler and Thomas Robinson and Tianhe Li and Tianjun Zhang and Tim Matthews and Timothy Chou and Tzook Shaked and Varun Vontimitta and Victoria Ajayi and Victoria Montanez and Vijai Mohan and Vinay Satish Kumar and Vishal Mangla and Vlad Ionescu and Vlad Poenaru and Vlad Tiberiu Mihailescu and Vladimir Ivanov and Wei Li and Wenchen Wang and Wenwen Jiang and Wes Bouaziz and Will Constable and Xiaocheng Tang and Xiaojian Wu and Xiaolan Wang and Xilun Wu and Xinbo Gao and Yaniv Kleinman and Yanjun Chen and Ye Hu and Ye Jia and Ye Qi and Yenda Li and Yilin Zhang and Ying Zhang and Yossi Adi and Youngjin Nam and Yu and Wang and Yu Zhao and Yuchen Hao and Yundi Qian and Yunlu Li and Yuzi He and Zach Rait and Zachary DeVito and Zef Rosnbrick and Zhaoduo Wen and Zhenyu Yang and Zhiwei Zhao and Zhiyu Ma},
      year={2024},
      eprint={2407.21783},
      archivePrefix={arXiv},
      primaryClass={cs.AI},
      url={https://arxiv.org/abs/2407.21783}, 
}

@misc{merity2016pointersentinelmixturemodels,
      title={Pointer Sentinel Mixture Models}, 
      author={Stephen Merity and Caiming Xiong and James Bradbury and Richard Socher},
      year={2016},
      eprint={1609.07843},
      archivePrefix={arXiv},
      primaryClass={cs.CL},
      url={https://arxiv.org/abs/1609.07843}, 
}

@misc{hendrycks2021measuringmassivemultitasklanguage,
      title={Measuring Massive Multitask Language Understanding}, 
      author={Dan Hendrycks and Collin Burns and Steven Basart and Andy Zou and Mantas Mazeika and Dawn Song and Jacob Steinhardt},
      year={2021},
      eprint={2009.03300},
      archivePrefix={arXiv},
      primaryClass={cs.CY},
      url={https://arxiv.org/abs/2009.03300}, 
}

@misc{zellers2019hellaswagmachinereallyfinish,
      title={HellaSwag: Can a Machine Really Finish Your Sentence?}, 
      author={Rowan Zellers and Ari Holtzman and Yonatan Bisk and Ali Farhadi and Yejin Choi},
      year={2019},
      eprint={1905.07830},
      archivePrefix={arXiv},
      primaryClass={cs.CL},
      url={https://arxiv.org/abs/1905.07830}, 
}

@misc{sakaguchi2019winograndeadversarialwinogradschema,
      title={WinoGrande: An Adversarial Winograd Schema Challenge at Scale}, 
      author={Keisuke Sakaguchi and Ronan Le Bras and Chandra Bhagavatula and Yejin Choi},
      year={2019},
      eprint={1907.10641},
      archivePrefix={arXiv},
      primaryClass={cs.CL},
      url={https://arxiv.org/abs/1907.10641}, 
}

@misc{lin2022truthfulqameasuringmodelsmimic,
      title={TruthfulQA: Measuring How Models Mimic Human Falsehoods}, 
      author={Stephanie Lin and Jacob Hilton and Owain Evans},
      year={2022},
      eprint={2109.07958},
      archivePrefix={arXiv},
      primaryClass={cs.CL},
      url={https://arxiv.org/abs/2109.07958}, 
}

@misc{clark2018thinksolvedquestionanswering,
      title={Think you have Solved Question Answering? Try ARC, the AI2 Reasoning Challenge}, 
      author={Peter Clark and Isaac Cowhey and Oren Etzioni and Tushar Khot and Ashish Sabharwal and Carissa Schoenick and Oyvind Tafjord},
      year={2018},
      eprint={1803.05457},
      archivePrefix={arXiv},
      primaryClass={cs.AI},
      url={https://arxiv.org/abs/1803.05457}, 
}

@misc{cobbe2021trainingverifierssolvemath,
      title={Training Verifiers to Solve Math Word Problems}, 
      author={Karl Cobbe and Vineet Kosaraju and Mohammad Bavarian and Mark Chen and Heewoo Jun and Lukasz Kaiser and Matthias Plappert and Jerry Tworek and Jacob Hilton and Reiichiro Nakano and Christopher Hesse and John Schulman},
      year={2021},
      eprint={2110.14168},
      archivePrefix={arXiv},
      primaryClass={cs.LG},
      url={https://arxiv.org/abs/2110.14168}, 
}

@misc{bai2024longbenchbilingualmultitaskbenchmark,
      title={LongBench: A Bilingual, Multitask Benchmark for Long Context Understanding}, 
      author={Yushi Bai and Xin Lv and Jiajie Zhang and Hongchang Lyu and Jiankai Tang and Zhidian Huang and Zhengxiao Du and Xiao Liu and Aohan Zeng and Lei Hou and Yuxiao Dong and Jie Tang and Juanzi Li},
      year={2024},
      eprint={2308.14508},
      archivePrefix={arXiv},
      primaryClass={cs.CL},
      url={https://arxiv.org/abs/2308.14508}, 
}

@inproceedings{lin-2004-rouge,
    title = "{ROUGE}: A Package for Automatic Evaluation of Summaries",
    author = "Lin, Chin-Yew",
    booktitle = "Text Summarization Branches Out",
    month = jul,
    year = "2004",
    address = "Barcelona, Spain",
    publisher = "Association for Computational Linguistics",
    url = "https://aclanthology.org/W04-1013/",
    pages = "74--81"
}

@misc{muennighoff2025olmoeopenmixtureofexpertslanguage,
      title={OLMoE: Open Mixture-of-Experts Language Models}, 
      author={Niklas Muennighoff and Luca Soldaini and Dirk Groeneveld and Kyle Lo and Jacob Morrison and Sewon Min and Weijia Shi and Pete Walsh and Oyvind Tafjord and Nathan Lambert and Yuling Gu and Shane Arora and Akshita Bhagia and Dustin Schwenk and David Wadden and Alexander Wettig and Binyuan Hui and Tim Dettmers and Douwe Kiela and Ali Farhadi and Noah A. Smith and Pang Wei Koh and Amanpreet Singh and Hannaneh Hajishirzi},
      year={2025},
      eprint={2409.02060},
      archivePrefix={arXiv},
      primaryClass={cs.CL},
      url={https://arxiv.org/abs/2409.02060}, 
}

@misc{su2023roformerenhancedtransformerrotary,
      title={RoFormer: Enhanced Transformer with Rotary Position Embedding}, 
      author={Jianlin Su and Yu Lu and Shengfeng Pan and Ahmed Murtadha and Bo Wen and Yunfeng Liu},
      year={2023},
      eprint={2104.09864},
      archivePrefix={arXiv},
      primaryClass={cs.CL},
      url={https://arxiv.org/abs/2104.09864}, 
}

@ARTICLE{1102314,
  author={Klema, V. and Laub, A.},
  journal={IEEE Transactions on Automatic Control}, 
  title={The singular value decomposition: Its computation and some applications}, 
  year={1980},
  volume={25},
  number={2},
  pages={164-176},
  doi={10.1109/TAC.1980.1102314}}
\bibliographystyle{tmlr}

\newpage

\appendix

\section{Appendix}

\subsection{Theoretical Justifications for Lossless Rotation}
\label{appendix:proofs_lossless_rotation}

This appendix provides the formal proofs for the claim that our projection-based rotation of the attention components is lossless prior to the pruning step. We demonstrate that both the attention scores and the final output of the attention mechanism remain unchanged.

\begin{lemma}[Rotational Invariance of Attention Scores]
Let $P_{QK} \in \mathbb{R}^{d_h \times d_h}$ be an orthogonal projection matrix derived from offline calibration. Let the projected query and key vectors be $\hat{q}_{i+1} = q_{i+1}P_{QK}$ and $\widehat{K}_{cache} = K_{cache}P_{QK}$. The attention scores computed using the original vectors ($S$) are identical to those computed using the projected vectors ($\hat{S}$).
\end{lemma}
\begin{proof}
The original attention scores are given by the dot product $S = q_{i+1}K_{cache}^T$. The attention scores computed with the projected vectors are $\hat{S} = \hat{q}_{i+1}\widehat{K}_{cache}^T$.

We can show that $\hat{S}$ is equivalent to $S$:
\begin{align*}
\hat{S} &= (\hat{q}_{i+1})(\widehat{K}_{cache})^T \\
&= (q_{i+1}P_{QK})(K_{cache}P_{QK})^T && \text{[Substituting definitions]} \\
&= (q_{i+1}P_{QK})(P_{QK}^T K_{cache}^T) && \text{[Using the transpose property $(AB)^T = B^T A^T$]} \\
&= q_{i+1}(P_{QK}P_{QK}^T)K_{cache}^T && \text{[Associativity of matrix multiplication]} \\
&= q_{i+1} I K_{cache}^T && \text{[Since $P_{QK}$ is orthogonal, $P_{QK}P_{QK}^T = I$]} \\
&= q_{i+1}K_{cache}^T = S
\end{align*}
Thus, the scores are identical ($\hat{S} = S$). This proves that the projection of queries and keys is a lossless rotation of the coordinate space that preserves their dot product relationships.
\end{proof}

\begin{lemma}[Losslessness of Full Attention with Absorbed Weights]
Let the rotated attention output for a head $j$ be $\hat{o}_j = \text{softmax}(S_j/\sqrt{d_h}) \hat{V}_j$, where $\hat{V}_j = V_j P_{VO}^{(j)}$. Let the modified output weight be $\widehat{W}_O^{(j)} = (P_{VO}^{(j)})^T W_O^{(j)}$. The final contribution to the MHA output from this head using rotated components is identical to the original.
\end{lemma}
\begin{proof}
The contribution of head $j$ to the final MHA output in the original model is $o_j W_O^{(j)}$. In our modified model, this contribution is $\hat{o}_j \widehat{W}_O^{(j)}$.

From Lemma 1, we know that the attention scores are unchanged by the $P_{QK}$ projection, so the softmax output is also unchanged. Let $A_j = \text{softmax}(S_j/\sqrt{d_h})$.

The original head output is $o_j = A_j V_j$. The rotated head output is $\hat{o}_j = A_j \hat{V}_j = A_j (V_j P_{VO}^{(j)})$.

Now, we can show the equivalence of the final output contribution:
\begin{align*}
\hat{o}_j \widehat{W}_O^{(j)} &= (A_j (V_j P_{VO}^{(j)})) (\widehat{W}_O^{(j)}) && \text{[Substituting definition of $\hat{o}_j$]} \\
&= (A_j V_j P_{VO}^{(j)}) ((P_{VO}^{(j)})^T W_O^{(j)}) && \text{[Substituting definition of $\widehat{W}_O^{(j)}$]} \\
&= A_j V_j (P_{VO}^{(j)}(P_{VO}^{(j)})^T) W_O^{(j)} && \text{[Associativity of matrix multiplication]} \\
&= A_j V_j I W_O^{(j)} && \text{[Since $P_{VO}$ is orthogonal, $PP^T = I$]} \\
&= (A_j V_j) W_O^{(j)} \\
&= o_j W_O^{(j)}
\end{align*}
This proves that applying the absorbed weights $\widehat{W}_V$ (to generate $\hat{V}$) and $\widehat{W}_O$ yields the exact same output as the original attention mechanism. Therefore, the only source of approximation error in our method comes from the subsequent pruning of dimensions.
\end{proof}

\subsection{Derivation of Computational Complexity}
\label{app:complexity_derivation}
This appendix provides the detailed mathematical derivation for the computational complexity analysis presented in Section~\ref{sec:complexity_analysis_main}.

\begin{proposition}[Complexity of Standard Attention]
The computational cost of the full attention calculation in standard auto-regressive attention for a single head at a sequence length $L$ is $C_{std} \approx 4 L \cdot d_{h}$ FLOPs.
\end{proposition}
\begin{proof}
The total cost consists of two main matrix-vector products and the softmax operation.
\begin{enumerate}
    \item \textbf{Score Calculation:} The product $qK_{cache}^T$ between the query ($1 \times d_h$) and the key cache ($L \times d_h$) costs approximately $2 L \cdot d_h$ FLOPs.
    \item \textbf{Softmax:} The softmax operation has a cost of $O(L)$, which is a lower-order term.
    \item \textbf{Output Calculation:} The product of the attention weights ($1 \times L$) and the value cache ($L \times d_h$) costs another $2 L \cdot d_h$ FLOPs.
\end{enumerate}
Summing these, the total complexity is dominated by the two matrix-vector products, giving $C_{std} \approx 4 L \cdot d_{h}$ FLOPs.
\end{proof}

\begin{proposition}[Complexity of SWAN]
The computational cost of the full attention calculation in SWAN for a single head at a sequence length $L$ is $C_{SWAN} \approx 4d_h^2 + 4(L-b)k_{active} + 4bd_h$.
\end{proposition}
\begin{proof}
The total cost is the sum of the runtime projections and the two hybrid matrix-vector products for scores and outputs, correctly accounting for sparsity in both the Key and Value caches.
\begin{enumerate}
    \item \textbf{Runtime Projections:} Projecting the current query and key with $P_{QK}$ is a fixed overhead costing $2 \times (2d_h^2) = 4d_h^2$ FLOPs.
    \item \textbf{Score Calculation:} The computation $S = \hat{q}_{i+1}K_{cache}^T$ involves a sparse-dense product (with the sparse part of the key cache) costing $2(L-b)k_{active}$ FLOPs and a dense-dense product (with the key buffer) costing $2bd_h$ FLOPs.
    \item \textbf{Softmax:} The cost is again a lower-order term of $O(L)$.
    \item \textbf{Output Calculation:} The final multiplication of the dense attention scores ($1 \times L$) with the hybrid value cache $V_{cache}$ has the same structure and cost as the score calculation. It involves a sparse-dense product with the sparse value cache costing $2(L-b)k_v$ (where $k_v = k_{active}$) FLOPs and a dense-dense product with the value buffer costing $2bd_h$ FLOPs.
\end{enumerate}
Summing these costs gives the total complexity:
\begin{align*}
C_{SWAN} &\approx 4d_h^2 + [2(L-b)k_{active} + 2bd_h] + [2(L-b)k_{active} + 2bd_h] \\
&\approx 4d_h^2 + 4(L-b)k_{active} + 4bd_h
\end{align*}
\end{proof}

\begin{proposition}[Computational Break-Even Point]
SWAN is computationally more efficient than standard attention when the sequence length $L$ satisfies: $L > \frac{d_h^2}{d_h - k_{active}} + b$.
\end{proposition}
\begin{proof}
We find the condition for which $C_{SWAN} < C_{std}$, ignoring lower-order terms:
\begin{align*}
4d_h^2 + 4(L-b)k_{active} + 4bd_h &< 4Ld_h \\
d_h^2 + (L-b)k_{active} + bd_h &< Ld_h \\
d_h^2 + Lk_{active} - bk_{active} + bd_h &< Ld_h \\
d_h^2 + b(d_h - k_{active}) &< Ld_h - Lk_{active} \\
d_h^2 + b(d_h - k_{active}) &< L(d_h - k_{active})
\end{align*}
Assuming $k_{active} < d_h$ (the practical use case), we can divide by the positive term $(d_h - k_{active})$:
\begin{align*}
\frac{d_h^2}{d_h - k_{active}} + b &< L
\end{align*}
This concludes the derivation.
\end{proof}

\subsubsection{Numerical Examples of the Break-Even Point}
To provide a more concrete understanding of the break-even formula, we analyze it under different scenarios. We assume a typical head dimension of $d_h = 128$, which makes the fixed overhead term $d_h^2 = 16,384$.

\textbf{Case 1: No Buffer ($b=0$)}
In this scenario, every token past the first is immediately pruned and stored sparsely.
\begin{itemize}
    \item \textbf{Aggressive Pruning (75\% pruned, $k_{active}=32$):} The per-token saving is proportional to $128 - 32 = 96$. The break-even point is $L > \frac{16,384}{96} + 0 \approx 171$ tokens.
    \item \textbf{Moderate Pruning (50\% pruned, $k_{active}=64$):} The per-token saving is proportional to $128 - 64 = 64$. The break-even point is $L > \frac{16,384}{64} + 0 = 256$ tokens.
    \item \textbf{Light Pruning (25\% pruned, $k_{active}=96$):} The per-token saving is proportional to $128 - 96 = 32$. The break-even point is $L > \frac{16,384}{32} + 0 = 512$ tokens.
\end{itemize}

\textbf{Case 2: With Buffer ($b=128$)}
Here, we maintain a dense buffer of the 128 most recent tokens, a common configuration in our experiments. The buffer size is simply added to the break-even point calculated from the pruning ratio.
\begin{itemize}
    \item \textbf{Aggressive Pruning (75\% pruned, $k_{active}=32$):} The break-even point is $L > 171 + 128 = 299$ tokens.
    \item \textbf{Moderate Pruning (50\% pruned, $k_{active}=64$):} The break-even point is $L > 256 + 128 = 384$ tokens.
    \item \textbf{Light Pruning (25\% pruned, $k_{active}=96$):} The break-even point is $L > 512 + 128 = 640$ tokens.
\end{itemize}
These examples clearly illustrate the trade-off: more aggressive pruning leads to substantial per-token savings, allowing the system to overcome the fixed projection overhead much earlier in the sequence. The dense buffer adds a constant offset to this point, representing the initial context length during which all attention computations remain dense.

\begin{table}[h!]
\centering
\caption{Performance of \texttt{Llama-3.1-8B-Instruct} as a function of the top-k retention ratio ($k_{active}/d_h$). A ratio of 1.0 (B) is the uncompressed baseline. The best performance for each task is highlighted in bold. Acronyms: HS (HellaSwag), WN (Winogrande), TQA (TruthfulQA), ARC-C (ARC Challenge), WT (WikiText). Arrows indicate whether a higher ($\uparrow$) or lower ($\downarrow$) score is better.}
\label{tab:perf_vs_k}
\begin{tabular}{c|cccccc|c|c}
\toprule
\textbf{Ratio} & \textbf{MMLU} $\uparrow$ & \textbf{GSM8K} $\uparrow$ & \textbf{HS} $\uparrow$ & \textbf{WN} $\uparrow$ & \textbf{TQA} $\uparrow$ & \textbf{ARC-C} $\uparrow$ & \textbf{WT} $\downarrow$ & \textbf{Avg Perf.} $\uparrow$ \\
\midrule
1.0 (B) & \textbf{0.687} & \textbf{0.804} & \textbf{0.608} & \textbf{0.762} & \textbf{0.551} & \textbf{0.621} & \textbf{8.912} & \textbf{0.671} \\
0.9 & 0.687 & 0.781 & 0.607 & 0.757 & 0.550 & 0.617 & 8.913 & 0.666 \\
0.75 & 0.683 & 0.792 & 0.606 & 0.762 & 0.551 & 0.611 & 8.943 & 0.668 \\
0.5 & 0.659 & 0.642 & 0.596 & 0.735 & 0.536 & 0.583 & 9.462 & 0.625 \\
0.3 & 0.310 & 0.038 & 0.456 & 0.545 & 0.491 & 0.391 & 21.097 & 0.372 \\
\bottomrule
\end{tabular}
\end{table}

\subsection{Benchmark Details}
\label{appendix:benchmarks}
This section provides details on the evaluation benchmarks used throughout our experiments, including the number of few-shot examples and the primary evaluation metric for each task.

\subsubsection*{Multi-Task Benchmarks}
\begin{itemize}
    \item \textbf{MMLU} (Massive Multitask Language Understanding) \citep{hendrycks2021measuringmassivemultitasklanguage}: A diverse benchmark testing world knowledge and problem-solving skills across 57 subjects. Evaluated with 5-shot prompting and measured by accuracy.
    \item \textbf{HellaSwag} \citep{zellers2019hellaswagmachinereallyfinish}: A commonsense reasoning task that involves choosing the most plausible continuation of a sentence. Evaluated with 10-shot prompting and measured by accuracy.
    \item \textbf{Winogrande} \citep{sakaguchi2019winograndeadversarialwinogradschema}: A benchmark designed to test commonsense reasoning through pronoun resolution problems. Evaluated with 5-shot prompting and measured by accuracy.
    \item \textbf{TruthfulQA (MC2)} \citep{lin2022truthfulqameasuringmodelsmimic}: A task that measures a model's tendency to answer questions truthfully, even when a common misconception provides a tempting alternative. Evaluated with 6-shot prompting and measured by accuracy.
    \item \textbf{ARC Challenge} \citep{clark2018thinksolvedquestionanswering}: A question-answering dataset composed of challenging science questions from grade-school to high-school level. Evaluated with 25-shot prompting and measured by accuracy.
    \item \textbf{WikiText} \citep{merity2016pointersentinelmixturemodels}: A benchmark measuring language modeling quality through zero-shot evaluation, measured by word-level perplexity, where lower is better.
\end{itemize}

\subsubsection*{Reasoning Benchmark}
\begin{itemize}
    \item \textbf{GSM8K} (Grade School Math 8K) \citep{cobbe2021trainingverifierssolvemath}: A dataset of grade-school math word problems that require multi-step reasoning. Evaluated with 5-shot prompting and measured by the flexible extract version of exact match of the final answer.
\end{itemize}

\subsubsection*{Long-Context Benchmarks (LongBench v1 \text{\citep{bai2024longbenchbilingualmultitaskbenchmark}})}
\begin{itemize}
    \item \textbf{Multi-News}: A long-document summarization task. Performance is measured by ROUGE score \citep{lin-2004-rouge}.
    \item \textbf{TREC}: A fine-grained question classification task on long documents. Performance is measured by classification score.
    \item \textbf{LCC} (Code Completion): A task evaluating a model's ability to complete long code snippets. Performance is measured by code similarity score.
    \item \textbf{Passage Retrieval (en)}: A task that requires retrieving relevant passages from a long document. Performance is measured by retrieval score.
    \item \textbf{SAMsum}: A dialogue summarization task. Performance is measured by ROUGE score.
\end{itemize}

\subsection{Performance vs. Pruning Ratio Analysis}
\label{app:perf_vs_pruning}
To establish the trade-off between KV-cache compression and model performance, we evaluate the \texttt{Llama-3.1-8B-Instruct} \citep{grattafiori2024llama3herdmodels} model's performance degradation as we vary the percentage of dimensions retained ($k_{active}$) in the sparse cache. The results, shown in Table~\ref{tab:perf_vs_k}, benchmark our method across several common-sense reasoning, knowledge, and language modeling tasks. A retention ratio of 1.0 represents the uncompressed baseline performance.

The results demonstrate that performance remains remarkably stable with moderate pruning. Retaining 75\% of the dimensions (`Ratio=0.75') results in a negligible drop in average performance, staying within 1\% of the baseline. Even when pruning half of the dimensions (`Ratio=0.5'), the model retains strong performance across most tasks, with an average degradation of less than 5\%. 

However, performance begins to degrade more significantly with more aggressive pruning. At a retention ratio of 0.3 (a 70\% reduction in dimensions), we observe a sharp decline across all benchmarks. Notably, the GSM8K task, which evaluates mathematical reasoning, is the most sensitive to information loss. Its performance collapses almost completely at the 0.3 ratio, dropping from over 80\% to just 3.8\% accuracy. This high sensitivity makes GSM8K an excellent benchmark for stress-testing our method. Consequently, our analysis in the main paper begins on this task, as its performance can be considered a practical lower bound for our approach's capabilities.

\subsection{Ablation Study on Key and Value Pruning Ratios}
\label{appendix:ablation}
To understand the relative importance of the key and value vectors in preserving model performance, we conducted an ablation study where we varied the proportion of dimensions retained for each. The study was performed using the \texttt{meta-llama/Llama-3.1-8B-Instruct} model with a dense buffer size of zero ($b=0$) to isolate the effect of the sparse cache. We evaluated the model on a suite of standard benchmarks, each with a specific few-shot setting to ensure fair comparison.

The results of this study are presented in Table~\ref{tab:ablation_kv}. We vary the retention ratio for keys ($TopK_R$) and values ($TopV_R$) such that their sum is always 1.0, representing a fixed information budget distributed between them.

\begin{table}[t!]
\centering
\caption{Ablation study on the \texttt{meta-llama/Llama-3.1-8B-Instruct} model with a zero-token buffer ($b=0$), varying the retention ratios for Key ($TopK_R$) and Value ($TopV_R$) vectors. The best performance for each task is highlighted in bold. Acronyms: MMLU (Massive Multitask Language Understanding), HS (HellaSwag), WN (Winogrande), TQA (TruthfulQA MC2), ARC-C (ARC Challenge), and WT (WikiText). Arrows indicate whether a higher ($\uparrow$) or lower ($\downarrow$) score is better.}
\label{tab:ablation_kv}
\begin{tabular}{cc|ccccc|c}
\toprule
\textbf{$TopK_R$} & \textbf{$TopV_R$} & \textbf{MMLU} $\uparrow$ & \textbf{HS} $\uparrow$ & \textbf{WN} $\uparrow$ & \textbf{TQA} $\uparrow$ & \textbf{ARC-C} $\uparrow$ & \textbf{WT} $\downarrow$ \\
\midrule
0.1 & 0.9 & 0.23 & 0.26 & 0.50 & 0.48 & 0.21 & 1431.04 \\
0.2 & 0.8 & 0.23 & 0.30 & 0.53 & 0.47 & 0.24 & 47.09 \\
0.3 & 0.7 & 0.50 & 0.55 & 0.59 & 0.47 & 0.49 & 13.01 \\
0.4 & 0.6 & 0.63 & 0.59 & 0.68 & 0.52 & 0.57 & 10.02 \\
0.5 & 0.5 & \textbf{0.66} & \textbf{0.60} & \textbf{0.73} & \textbf{0.54} & \textbf{0.58} & \textbf{9.46} \\
0.6 & 0.4 & \textbf{0.66} & 0.59 & \textbf{0.73} & 0.52 & \textbf{0.58} & 9.52 \\
0.7 & 0.3 & 0.64 & 0.58 & 0.72 & 0.50 & 0.56 & 10.11 \\
0.8 & 0.2 & 0.57 & 0.52 & 0.66 & 0.50 & 0.49 & 12.48 \\
0.9 & 0.1 & 0.24 & 0.32 & 0.52 & 0.48 & 0.26 & 63.30 \\
\bottomrule
\end{tabular}
\end{table}

The results clearly indicate that both key and value vectors are critical for retaining model performance. Extreme pruning of either component leads to a dramatic drop in accuracy across all tasks. For instance, retaining only 10\% of key dimensions ($TopK_R=0.1$) while keeping 90\% of value dimensions results in a catastrophic increase in WikiText perplexity and poor performance on all other benchmarks.

Interestingly, the optimal balance appears to be near the center. The configuration with symmetric pruning ($TopK_R=0.5, TopV_R=0.5$) achieves the best or near-best results on every single task, suggesting that, as a general rule, keys and values are of roughly equal importance. The slightly asymmetric configuration of ($TopK_R=0.6, TopV_R=0.4$) also performs exceptionally well, matching the top performance on MMLU, Winogrande, and ARC-C. This might suggest that for some tasks, retaining slightly more information in the key vectors, which are responsible for the attention score distribution, is marginally more beneficial than retaining it in the value vectors, which carry the content to be aggregated. However, the overall trend strongly supports a balanced pruning strategy as a robust and effective baseline for our method.

\subsection{Ablation Study: Importance of Projection Matrix Specificity}
While any orthogonal projection is mathematically lossless before pruning, the performance of our method hinges on the quality of the rotation itself. A well-chosen rotation aligns the most important information with the top dimensions, minimizing information loss during the subsequent pruning step. To validate that our data-driven method for computing projection matrices is critical for performance, we conduct a series of ablation studies. We compare our proposed projection method against several variants, all evaluated with a 50\% retention ratio ($k_{active}/d_h = 0.5$).

The ablation experiments are designed as follows:
\begin{itemize}
    \item \textbf{Random Projection:} We replace our learned matrices with an orthogonal matrix derived from 4096 randomly generated vectors following a Gaussian distribution. This tests whether a generic, non-data-driven orthogonal basis is sufficient.
    \item \textbf{Layer-Shuffle:} We randomly shuffle our pre-computed projection matrices across the model's layers. This experiment is designed to determine if the learned subspaces are specific to each layer or if a more generic, layer-agnostic projection would suffice.
    \item \textbf{KV-Shuffle:} We interchange the projection matrices for the Key-Value and Query-Output subspaces ($P_{QK} \leftrightarrow P_{VO}$). This tests whether the learned subspaces for these distinct components are interchangeable or highly specialized.
    \item \textbf{Head-Shuffle:} Within each layer, we randomly shuffle the projection matrices among the different attention heads. This evaluates if the learned subspaces are specific to individual heads or are generalizable across all heads in a layer.
\end{itemize}

\begin{table}[t!]
\centering
\caption{Ablation results comparing our proposed projection method against several variants at a 50\% retention ratio. Our proposed method consistently outperforms all others, highlighting the importance of its data-driven and component-specific nature. Acronyms: HS (HellaSwag), WN (Winogrande), TQA (TruthfulQA), ARC-C (ARC Challenge), WT (WikiText).}
\label{tab:ablation_projection}
\begin{tabular}{l|cccccc|c}
\toprule
\textbf{Projection Method} & \textbf{MMLU} $\uparrow$ & \textbf{HS} $\uparrow$ & \textbf{WN} $\uparrow$ & \textbf{TQA} $\uparrow$ & \textbf{ARC-C} $\uparrow$ & \textbf{WT} $\downarrow$ & \textbf{Avg Perf.} $\uparrow$ \\
\midrule
\textbf{Our Projection} & \textbf{0.66} & \textbf{0.60} & \textbf{0.73} & \textbf{0.54} & \textbf{0.58} & \textbf{9.46} & \textbf{0.62} \\
Head-Shuffle & 0.60 & 0.57 & 0.70 & 0.50 & 0.56 & 10.72 & 0.59 \\
Layer-Shuffle & 0.59 & 0.57 & 0.68 & 0.51 & 0.56 & 10.85 & 0.58 \\
KV-Shuffle & 0.58 & 0.57 & 0.69 & 0.49 & 0.54 & 11.24 & 0.57 \\
Random Projection & 0.57 & 0.57 & 0.68 & 0.49 & 0.54 & 11.13 & 0.57 \\

\bottomrule
\end{tabular}
\end{table}

The results in Table~\ref{tab:ablation_projection} unequivocally demonstrate the superiority of our tailored projection method. Our approach outperforms all ablation variants across every benchmark, confirming that the specificity of the projection is crucial for retaining model performance.

\textbf{Key Findings:}
\begin{itemize}
    \item The \textbf{Random Projection} yields the most significant performance degradation, proving that the projection matrix must be data-driven and derived from the model's actual activation patterns. A generic orthogonal basis is insufficient for identifying and preserving salient information.
    \item The performance drops from \textbf{shuffling across layers, heads, and KV-components} confirm that the learned low-dimensional subspaces are highly specialized. This validates our approach of creating distinct projection matrices for each specific component (per layer, per head, and for QK/VO subspaces separately), as these structures are not interchangeable.
    \item Even the seemingly minor drop in the \textbf{Head-Shuffle} experiment highlights the fine-grained nature of attention. Each head learns to focus on different aspects of the input, and our method successfully captures this specialized structure.
\end{itemize}
In conclusion, this ablation study validates that the performance of our method is not merely a consequence of using an orthogonal transformation but is critically dependent on our careful, data-driven process for constructing matrices that are specific to each component of the attention mechanism.

\subsection{Detailed Results on Additional NLP Benchmarks}
\label{appendix:nlp_details}
This section provides the detailed performance plots for the NLP benchmarks not included in the main body of the paper, evaluated on both \texttt{Llama-3.1-8B-Instruct} (top row) and \texttt{OLMoE-1B-7B-0924-Instruct} (bottom row). The results on commonsense reasoning (Winogrande, HellaSwag), model truthfulness (TruthfulQA), and language modeling (WikiText) are fully consistent with the findings discussed in Section~\ref{sec:nlp_benchmarks}.

\begin{figure}[t!]
    \centering
    \includegraphics[width=\textwidth]{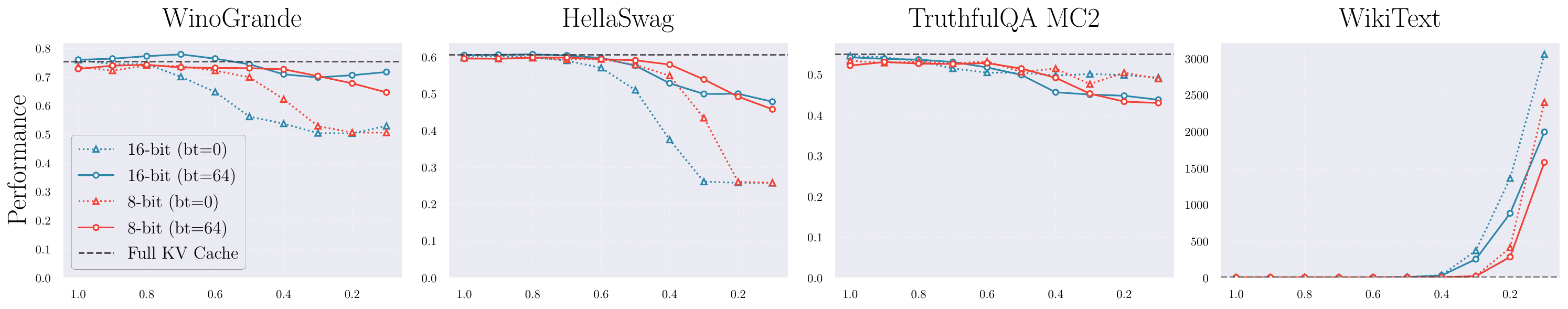}
    \includegraphics[width=\textwidth]{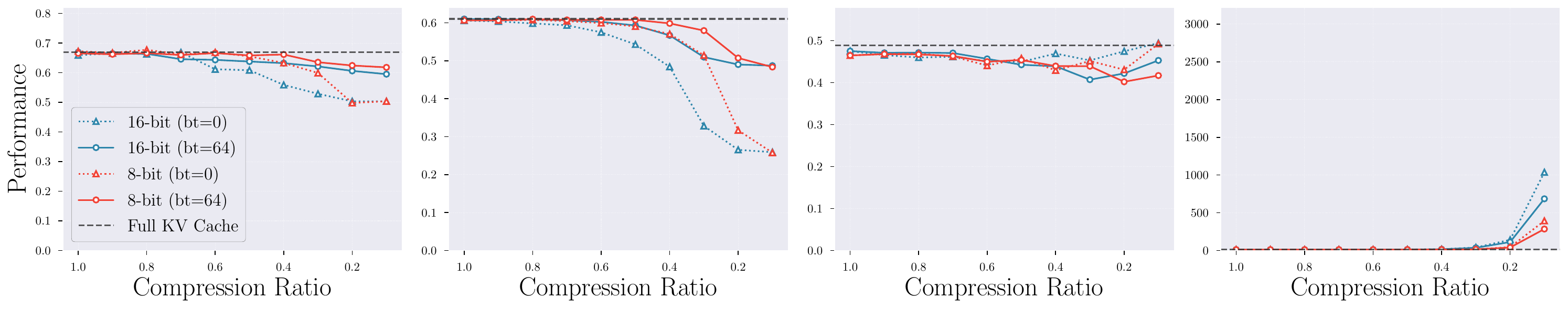}
    \caption{Detailed performance on additional NLP tasks for \texttt{Llama-3.1-8B-Instruct} (top row) and \texttt{OLMoE-1B-7B-0924-Instruct} (bottom row). The figure displays results for Winogrande, HellaSwag, TruthfulQA MC2, and WikiText. The trends confirm the critical role of the dense buffer in preserving performance across different task types and model architectures.}
    \label{fig:nlp_appendix_details}
\end{figure}

The detailed results in Figure~\ref{fig:nlp_appendix_details} not only reinforce our main findings but also reveal fascinating, nuanced interactions between SWAN, model architecture, and task-specific requirements. A key insight emerges from comparing the Multi-Head Attention (MHA) of OLMoE with the Grouped-Query Attention (GQA) of Llama. OLMoE's architecture, with unique Key and Value projections for each head, is inherently more sparse than Llama's, where dense Key and Value vectors are shared across groups of queries. As SWAN is designed to directly exploit sparsity, it naturally exhibits a more graceful performance degradation on the sparser OLMoE architecture, a compelling trend observed across all tasks.

On the commonsense reasoning benchmarks, SWAN demonstrates exceptional resilience. For \textbf{Winogrande}, performance remains remarkably stable, showing very little drop even at the highest compression levels. Even the zero-buffer versions maintain their integrity until a 60\% compression ratio, highlighting the task's robustness to information pruning. For \textbf{HellaSwag}, the buffered variants show almost no performance loss down to a 40\% compression ratio, after which a clear threshold is crossed and performance drops sharply. This suggests a critical point of information density required for this specific reasoning task.

The \textbf{TruthfulQA MC2} benchmark presents a unique case. Both buffered and non-buffered variants are surprisingly stable, with only a minor, consistent drop after a 50\% compression ratio. This indicates that the information required for factual consistency is highly concentrated within the most energetic components of the KV vectors and is less dependent on the high-fidelity recent context provided by the buffer.

Finally, the highly sensitive \textbf{WikiText} perplexity benchmark serves as a powerful testament to our method's architectural advantages. While performance holds steady until a 40\% compression ratio for Llama and 30\% for OLMoE, the subsequent degradation is dramatically different. The perplexity spike on OLMoE is three times less severe than on Llama. This is a crucial piece of evidence supporting our central claim: SWAN’s design, which thrives on sparsity, is inherently more efficient and less disruptive on models with sparser attention mechanisms like MHA.

\subsection{Detailed LongBench Task Results}
\label{appendix:longbench_details}
This section provides detailed results for the remaining LongBench tasks, which confirm the conclusions drawn in the main paper. As shown in Figure~\ref{fig:longbench_appendix_details}, the trends are consistent across code completion (LCC), fine-grained classification (TREC), and passage retrieval.

Across all tasks, the 128-token buffer is essential for avoiding performance collapse, and the buffered SWAN variants exhibit a robust and graceful trade-off between compression and accuracy. The specific degradation patterns vary by task, suggesting different sensitivities to information loss. For instance, performance on the TREC classification task shows a particularly sharp drop after 50\% compression, indicating a high reliance on specific details that are pruned more aggressively at that point. Nonetheless, the overall effectiveness of our hybrid-cache approach is validated across these diverse, long-context challenges.

\begin{figure}[t!]
    \centering
    \includegraphics[width=\textwidth]{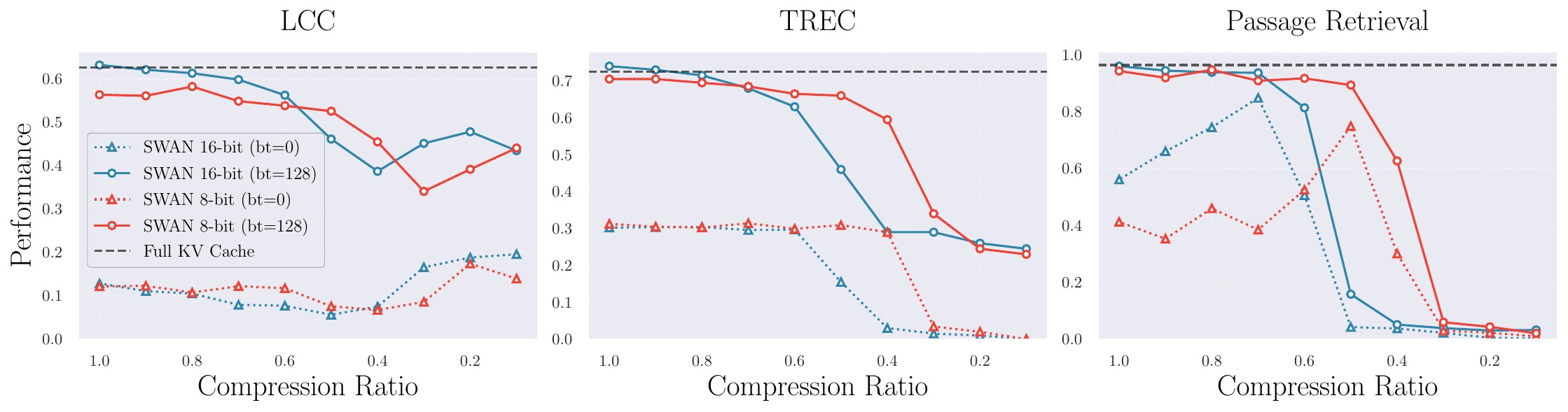}
    \caption{Performance on additional LongBench tasks. The buffered SWAN (`bt=128') confirms its effectiveness on code completion (LCC), classification (TREC), and passage retrieval, demonstrating the generalizability of our approach for diverse long-context scenarios.}
    \label{fig:longbench_appendix_details}
\end{figure}

\end{document}